\documentclass[oneside,11pt]{article}
\pdfoutput=1  % to ensure arXiv is compiling it as I do

\newif\ifmydraft
\mydrafttrue
\mydraftfalse

\newif\ifarXiv
\arXivtrue
\ifarXiv
  \usepackage[preprint]{jmlr2e}
\else
  \usepackage{jmlr2e}
\fi

\jmlrheading{x}{2021}{x}{x}{x}{x}{Johannes Lederer}

\ShortHeadings{Wide Neural Networks}{Johannes Lederer}
\firstpageno{1}

\usepackage{microtype}
\usepackage{graphicx}
\usepackage{subfigure}
\usepackage{booktabs}
\usepackage{xcolor}
\usepackage{amsfonts, amsmath}%, amsthm}%, amssymb}

\usepackage{tikz}
\usetikzlibrary{calc}

\ifmydraft
  \newcommand{\draftcolor}{purple}
  \newcommand{\draftcolorpage}{black!20}
\else
  \newcommand{\draftcolor}{black}
  \newcommand{\draftcolorpage}{white}
\fi

\ifmydraft\usepackage{showlabels}\else\fi

\ifmydraft
  \newcommand{\jl}[1]{\color{blue}#1\color{black}}
  \pagecolor{\draftcolorpage}
\else
  \newcommand{\jl}[1]{}
\fi

\newcommand{\tp}{^\top}
\newcommand{\R}{\textcolor{\draftcolor}{\mathbb{R}}}
\makeatletter 
\newcommand*{\deq}{\ensuremath{\mathrel{\rlap{%
\raisebox{0.3ex}{$\m@th\cdot$}}%
\raisebox{-0.3ex}{$\m@th\cdot$}}=}}
\makeatother

\DeclareMathOperator*{\argmin}{arg\,min}

\newcommand{\zero}{\textcolor{\draftcolor}{\boldsymbol{0}}}

\newcommand{\abs}[1]{\lvert#1\rvert}
\newcommand{\absB}[1]{\bigl\lvert#1\bigr\rvert}

\newcommand{\absBBB}[1]{\biggl\lvert#1\biggr\rvert}

\newcommand{\card}[1]{\#\{#1\}}

\newcommand{\norm}[1]{|\!|#1|\!|}

\newcommand{\normM}[1]{|\!|\!|#1|\!|\!|}
\newcommand{\normMB}[1]{\big|\!\big|\!\big|#1\big|\!\big|\!\big|}

\newcommand{\normsupM}[1]{\normM{#1}_\infty}

\newcommand{\normone}[1]{\norm{#1}_1}

\newcommand{\normoneM}[1]{\normM{#1}_1}
\newcommand{\normoneMB}[1]{\normMB{#1}_1}

\newcommand{\normqM}[1]{\normM{#1}_q}
\newcommand{\normqMB}[1]{\normMB{#1}_q}

\newcommand{\normtwos}[1]{\norm{#1}_2^2}

\newcommand{\nbrlayers}{\textcolor{\draftcolor}{\ensuremath{l}}}
\newcommand{\nbrsamples}{\textcolor{\draftcolor}{\ensuremath{n}}}
\newcommand{\nbrinput}{\textcolor{\draftcolor}{\ensuremath{d}}}
\newcommand{\nbroutput}{\textcolor{\draftcolor}{\ensuremath{m}}}

\newcommand{\nbrwidthmin}{\textcolor{\draftcolor}{\ensuremath{\underline{w}}}}

\newcommand{\nbrparameter}{\textcolor{\draftcolor}{\ensuremath{p}}}
\newcommand{\nbrparameterj}{\textcolor{\draftcolor}{\ensuremath{\nbrparameter^j}}}
\newcommand{\nbrparameterjj}{\textcolor{\draftcolor}{\ensuremath{\nbrparameter^{j+1}}}}
\newcommand{\nbrparameterl}{\textcolor{\draftcolor}{\ensuremath{\nbrparameter^{\nbrlayers}}}}
\newcommand{\nbrparameterll}{\textcolor{\draftcolor}{\ensuremath{\nbrparameter^{\nbrlayers+1}}}}

\newcommand{\parameterz}{\textcolor{\draftcolor}{\ensuremath{\parameterE^{0}}}}
\newcommand{\activation}{\textcolor{\draftcolor}{\ensuremath{\mathfrak{f}}}}

\newcommand{\activationind}{\textcolor{\draftcolor}{\ensuremath{\underline{\mathfrak{f}}}}}
\newcommand{\activationindj}{\textcolor{\draftcolor}{\ensuremath{\mathfrak{\activationind}^j}}}
\newcommand{\activationindjF}[1]{\textcolor{\draftcolor}{\ensuremath{\activationindj[#1]}}}

\newcommand{\vectorin}{\textcolor{\draftcolor}{\ensuremath{\boldsymbol{x}}}}
\newcommand{\vectorout}{\textcolor{\draftcolor}{\ensuremath{\boldsymbol{y}}}}
\newcommand{\vectorini}{\textcolor{\draftcolor}{\ensuremath{\boldsymbol{x}_i}}}
\newcommand{\vectorouti}{\textcolor{\draftcolor}{\ensuremath{\boldsymbol{y}_i}}}

\newcommand{\data}{\textcolor{\draftcolor}{\ensuremath{X}}}

\newcommand{\parameterS}{\textcolor{\draftcolor}{\ensuremath{\mathcal M}}}
\newcommand{\parameterSFull}{\textcolor{\draftcolor}{\ensuremath{\overline{\parameterS}}}}
\newcommand{\parameterE}{\textcolor{\draftcolor}{\ensuremath{\Theta}}}
\newcommand{\parameter}{\textcolor{\draftcolor}{\ensuremath{\boldsymbol{\parameterE}}}}
\newcommand{\parameterj}{\textcolor{\draftcolor}{\ensuremath{\Theta^j}}}
\newcommand{\parameterl}{\textcolor{\draftcolor}{\ensuremath{\Theta^{\nbrlayers}}}}

\newcommand{\parameterGE}{\textcolor{\draftcolor}{\ensuremath{\Gamma}}}
\newcommand{\parameterG}{\textcolor{\draftcolor}{\ensuremath{\boldsymbol{\parameterGE}}}}
\newcommand{\parameterGj}{\textcolor{\draftcolor}{\ensuremath{\parameterGE^j}}}
\newcommand{\parameterGl}{\textcolor{\draftcolor}{\ensuremath{\parameterGE^{\nbrlayers}}}}

\newcommand{\parameterGGE}{\textcolor{\draftcolor}{\ensuremath{\Psi}}}
\newcommand{\parameterGG}{\textcolor{\draftcolor}{\ensuremath{\boldsymbol{\parameterGGE}}}}

 \newcommand{\estimator}{\textcolor{\draftcolor}{\ensuremath{\widehat{\parameter}}}}

\newcommand{\parameterUE}{\textcolor{\draftcolor}{\ensuremath{\overline{\parameterE}}}}
\newcommand{\parameterU}{\textcolor{\draftcolor}{\ensuremath{\overline{\boldsymbol{\parameterE}}}}}
\newcommand{\parameterUl}{\textcolor{\draftcolor}{\ensuremath{\parameterUE^{\nbrlayers}}}}

\newcommand{\parameterUUE}{\textcolor{\draftcolor}{\ensuremath{\widetilde{\parameterE}}}}

\newcommand{\parameterUUUE}{\textcolor{\draftcolor}{\ensuremath{\dot\parameterE}}}

\newcommand{\parameterL}{\textcolor{\draftcolor}{\ensuremath{\underline{\parameter}}}}

\newcommand{\network}{\textcolor{\draftcolor}{\ensuremath{\mathfrak{g}}}}
\newcommand{\networkA}{\textcolor{\draftcolor}{\ensuremath{\network_{\parameter}}}}
\newcommand{\networkAF}[1]{\textcolor{\draftcolor}{\ensuremath{\networkA[#1]}}}
\newcommand{\networkAG}{\textcolor{\draftcolor}{\ensuremath{\network_{\parameterG}}}}
\newcommand{\networkAGF}[1]{\textcolor{\draftcolor}{\ensuremath{\networkAG[#1]}}}

\newcommand{\networkS}{\textcolor{\draftcolor}{\ensuremath{\mathcal{G}}}}
\newcommand{\networkSFull}{\textcolor{\draftcolor}{\ensuremath{\overline{\mathcal{G}}}}}

\newcommand{\loss}{\textcolor{\draftcolor}{\ensuremath{\mathfrak{l}}}}
\newcommand{\lossF}[1]{\textcolor{\draftcolor}{\ensuremath{\loss[#1]}}}
\newcommand{\lossFB}[1]{\textcolor{\draftcolor}{\ensuremath{\loss\bigl[#1\bigr]}}}
\newcommand{\lossFBB}[1]{\textcolor{\draftcolor}{\ensuremath{\loss\Bigl[#1\Bigr]}}}

\newcommand{\regularizer}{\textcolor{\draftcolor}{\ensuremath{\mathfrak{r}}}}
\newcommand{\regularizerF}[1]{\textcolor{\draftcolor}{\ensuremath{\regularizer[#1]}}}
\newcommand{\regularizerFB}[1]{\textcolor{\draftcolor}{\ensuremath{\regularizer\bigl[#1\bigr]}}}

\newcommand{\regularizerR}{\textcolor{\draftcolor}{\ensuremath{\mathfrak{\underline{r}}}}}
\newcommand{\regularizerRF}[1]{\textcolor{\draftcolor}{\ensuremath{\regularizerR[#1]}}}

\newcommand{\tuningparameterout}{\textcolor{\draftcolor}{\ensuremath{a_{\regularizer}}}}
\newcommand{\tuningparameterin}{\textcolor{\draftcolor}{\ensuremath{b_{\regularizer}}}}

\newcommand{\pathfunction}{\textcolor{\draftcolor}{\ensuremath{\mathfrak{h}}}}
\newcommand{\pathfunctionF}[1]{\textcolor{\draftcolor}{\ensuremath{\pathfunction[#1]}}}
\newcommand{\pathfunctionFB}[1]{\textcolor{\draftcolor}{\ensuremath{\pathfunction\bigl[#1\bigr]}}}

\newcommand{\pathfunctionFBBB}[1]{\textcolor{\draftcolor}{\ensuremath{\pathfunction\biggl[#1\biggr]}}}
\newcommand{\pathfunctionP}{\textcolor{\draftcolor}{\ensuremath{\tilde{\pathfunction}}}}
\newcommand{\pathfunctionPF}[1]{\textcolor{\draftcolor}{\ensuremath{\pathfunctionP[#1]}}}
\newcommand{\pcon}{\textcolor{\draftcolor}{\ensuremath{\mathfrak{\smile}}}}
\newcommand{\pequi}{\textcolor{\draftcolor}{\ensuremath{\mathfrak{\leftrightsquigarrow}}}}
\newcommand{\pdec}{\textcolor{\draftcolor}{\ensuremath{\mathfrak{\leftrightarrow}}}}

\newcommand{\shapefactor}{\textcolor{\draftcolor}{\ensuremath{s}}}
\newcommand{\shapelower}{\textcolor{\draftcolor}{\ensuremath{\mathcal L_{\shapefactor,\nbrlayers}}}}
\newcommand{\shapeupper}{\textcolor{\draftcolor}{\ensuremath{\mathcal U_{\shapefactor,\nbrlayers}}}}
\newcommand{\shapelowerE}{\textcolor{\draftcolor}{\ensuremath{\mathcal L}}}
\newcommand{\shapeupperE}{\textcolor{\draftcolor}{\ensuremath{\mathcal U}}}

\newcommand{\paradd}{\textcolor{\draftcolor}{\ensuremath{t}}}
\newcommand{\paraddo}{\textcolor{\draftcolor}{\ensuremath{\paradd_1}}}
\newcommand{\paraddt}{\textcolor{\draftcolor}{\ensuremath{\paradd_2}}}
\newcommand{\paraddw}{\textcolor{\draftcolor}{\ensuremath{a}}}
\newcommand{\paraddP}{\textcolor{\draftcolor}{\ensuremath{\bar\paradd}}}

\newcommand{\paraddPP}{\textcolor{\draftcolor}{\ensuremath{c}}}

\newcommand{\shfunction}{\textcolor{\draftcolor}{\ensuremath{\mathfrak{h}}}}
\newcommand{\shfunctionF}[1]{\textcolor{\draftcolor}{\ensuremath{\shfunction[#1]}}}
\newcommand{\shfunctionFB}[1]{\textcolor{\draftcolor}{\ensuremath{\shfunction\bigl[#1\bigr]}}}

\newcommand{\shfunctionFBBBB}[1]{\textcolor{\draftcolor}{\ensuremath{\shfunction\Biggl[#1\Biggr]}}}

\newcommand{\shmatrixout}{\textcolor{\draftcolor}{\ensuremath{A}}}
\newcommand{\shmatrixoutS}{\textcolor{\draftcolor}{\ensuremath{\dot{A}}}}
\newcommand{\shmatrixin}{\textcolor{\draftcolor}{\ensuremath{B}}}
\newcommand{\shmatrixdata}{\textcolor{\draftcolor}{\ensuremath{C}}}

\newcommand{\shmatrixoutU}{\textcolor{\draftcolor}{\ensuremath{\overline{A}}}}
\newcommand{\shmatrixinU}{\textcolor{\draftcolor}{\ensuremath{\overline{B}}}}

\newcommand{\shmatrixoutL}{\textcolor{\draftcolor}{\ensuremath{\underline{A}}}}
\newcommand{\shmatrixinL}{\textcolor{\draftcolor}{\ensuremath{\underline{B}}}}

\newcommand{\tlnbrout}{\textcolor{\draftcolor}{\ensuremath{u}}}
\newcommand{\tlnbrmed}{\textcolor{\draftcolor}{\ensuremath{v}}}
\newcommand{\tlnbrin}{\textcolor{\draftcolor}{\ensuremath{o}}}
\newcommand{\tlnbrdata}{\textcolor{\draftcolor}{\ensuremath{r}}}

\newcommand{\matrixG}{\textcolor{\draftcolor}{\ensuremath{M}}}

\newcommand{\caraweightE}{\textcolor{\draftcolor}{\ensuremath{t}}}
\newcommand{\caraweight}{\textcolor{\draftcolor}{\ensuremath{\boldsymbol{\caraweightE}}}}

\newcommand{\caraweightj}{\textcolor{\draftcolor}{\ensuremath{\caraweightE_j}}}

\newcommand{\caraweightGE}{\textcolor{\draftcolor}{\ensuremath{\tilde t}}}
\newcommand{\caraweightG}{\textcolor{\draftcolor}{\ensuremath{\boldsymbol{\caraweightGE}}}}

\newcommand{\caraweightGj}{\textcolor{\draftcolor}{\ensuremath{\caraweightGE_j}}}

\newcommand{\caravector}{\textcolor{\draftcolor}{\ensuremath{\boldsymbol{z}}}}

\newcommand{\caravectorj}{\textcolor{\draftcolor}{\ensuremath{\caravector_j}}}

\newcommand{\losspower}{\textcolor{\draftcolor}{\ensuremath{q}}}

\newcommand{\losspowerin}{\textcolor{\draftcolor}{\ensuremath{\losspower_{\shmatrixin}}}}
\newcommand{\losspowerout}{\textcolor{\draftcolor}{\ensuremath{\losspower_{\shmatrixout}}}}

\newcommand{\normloss}[1]{|\!|#1|\!|_{\losspower}}

\newcommand{\normlossout}[1]{|\!|#1|\!|_{\losspowerout}}
\newcommand{\normlossoutH}[1]{|\!|#1|\!|^{\losspowerout}_{\losspowerout}}

\newcommand{\activationdegree}{\textcolor{\draftcolor}{\ensuremath{k}}}

\newcommand{\domainin}{\textcolor{\draftcolor}{\ensuremath{\mathcal{D}_{\vectorin}}}}
\newcommand{\domainout}{\textcolor{\draftcolor}{\ensuremath{\mathcal{D}_{\vectorout}}}}
\newcommand{\domainoutFull}{\textcolor{\draftcolor}{\ensuremath{\R^{\nbroutput}}}}

\newcommand{\transita}{\textcolor{\draftcolor}{\ensuremath{c_1}}}
\newcommand{\transitb}{\textcolor{\draftcolor}{\ensuremath{c_2}}}

\newcommand{\transitP}{\textcolor{\draftcolor}{\ensuremath{[\parameter',\parameterG']_{\transita,\transitb}}}}

\newcommand{\transitpoP}{\textcolor{\draftcolor}{\ensuremath{[\parameter',\parameterG']_{1-\paraddo,\paraddo}}}}

\newcommand{\transitptP}{\textcolor{\draftcolor}{\ensuremath{[\parameter',\parameterG']_{1-\paraddt,\paraddt}}}}

\newcommand{\transitwP}{\textcolor{\draftcolor}{\ensuremath{[\parameter',\parameterG']_{1-(1-\paraddw)\paraddo-\paraddw\paraddt,(1-\paraddw)\paraddo+\paraddw\paraddt}}}}

\newcommand{\perm}{\textcolor{\draftcolor}{\ensuremath{\mathfrak{p}}}}
\newcommand{\permF}[1]{\textcolor{\draftcolor}{\ensuremath{\mathfrak{p}[#1]}}}

\newcommand{\permj}{\textcolor{\draftcolor}{\ensuremath{\perm^j}}}
\newcommand{\permjF}[1]{\textcolor{\draftcolor}{\ensuremath{\permj[#1]}}}

\newcommand{\permjj}{\textcolor{\draftcolor}{\ensuremath{\perm^{j+1}}}}
\newcommand{\permjjF}[1]{\textcolor{\draftcolor}{\ensuremath{\permjj[#1]}}}

\newcommand{\problematic}{\textcolor{\draftcolor}{confined point}}
\newcommand{\problematics}{\textcolor{\draftcolor}{confined points}}

\newcommand{\Problematics}{\textcolor{\draftcolor}{Confined points}}
\newcommand{\PRoblematics}{\textcolor{\draftcolor}{Confined Points}}

\newcommand{\unproblematic}{\textcolor{\draftcolor}{handleable}}

\newcommand{\bads}{\textcolor{\draftcolor}{problematic points}}

\newcommand{\figureminima}[1]{\begin{figure}[#1]
  \centering
  \begin{tikzpicture}[scale=1]
    \draw[-latex] (-0.1, 0) -- (4, 0);
    \draw[-latex] (0, -0.1) -- (0, 2);
   \node at (3.8, -0.2) {\scriptsize \parameter};
   \node at (0.4, 1.9) {\scriptsize $\lossF{\networkA}$};
    \draw[thick, line cap = round] plot [smooth] coordinates{(-0.1, 1.8) (0.4, 0.8) (1, 1.2) (1.4, 0.2) (1.8, 0.4) (2.2, 0.177) (2.4, 0.5)};
\draw[thick, line cap=round] (2.4, 0.5) -- (2.5, 0.5);
 \draw[thick] plot [smooth] coordinates{(2.5, 0.5) (2.8, 1.081) (3.5, 1.179) (4, 1.8)};
    \node at (1.8, 1.05) {\tiny global minima};
    \draw[-latex] (2.0, 0.9) -- (2.12, 0.35);
    \draw[-latex] (1.6, 0.9) -- (1.48, 0.35);
    \node at (-0.17, 0.3) {\tiny \problematic};
   \draw[-latex] (0.05, 0.45) -- (0.34, 0.73);
    \node at (3.85, 0.2) {\tiny ``\unproblematic'' local minima};
   \draw[-latex] (2.67, 0.25) -- (2.48, 0.45);
    \node at (3.25, 1.92) {\tiny saddle point};
   \draw[-latex] (3.25, 1.8) -- (3.16, 1.22);

  \end{tikzpicture}
  \caption{hypothetical objective function:
\problematics\ are typically very difficult to escape from, 
while other local minima and saddle points can often be handled by using momentum, stochastic approaches, and so forth}
  \label{fig:minima}
\end{figure}}
\newcommand{\figureshaping}[1]{
\begin{figure}[#1]
  \centering
  \raisebox{2mm}{\begin{tikzpicture}[scale=0.85]
    \filldraw[black!7!white, rounded corners=10pt] (0, 0) rectangle (8, 3);

\node at (7.65, 2.75) {\scriptsize $\parameterS$};
    \filldraw (1, 0.5) circle (1pt);
 \node (originalA) at (1, 0.5) {};
    \node at (1, 0.27) {\scriptsize ${\parameter}$};
    \filldraw (2, 2.5) circle (1pt);
    \filldraw (2.6, 2.56) circle (1pt);
    \node at (2.65, 2.85) {\scriptsize $\overline{\parameter}'$};
    \node at (2, 2.8) {\scriptsize ${\overline{\parameter}}$};
    \filldraw (5, 2) circle (1pt);
    \node at (5, 2.25) {\scriptsize ${\underline{\parameterG}}$};
    \filldraw (4.1, 1.5) circle (1pt);
    \node at (4.15, 1.26) {\scriptsize $\underline{\parameterG}'$};
    \filldraw (7, 1.5) circle (1pt);
    \node (originalB) at (7, 1.5) {};
    \node at (7, 1.25) {\scriptsize ${\parameterG}$};

   \draw[thick] plot [smooth] coordinates {(1, 0.5) (0.9, 1.8)  (2, 2.5)};

   \draw[thick] plot [smooth] coordinates {(5, 2) (6.3, 2.1)  (7, 1.5)};

   \draw[thick, dashed] plot [smooth] coordinates {(2, 2.5) (3, 2.5) (4, 1.5) (5, 2)};

\node at (1.6, 1.69) {\tiny Proposition~\ref{res:reparametrization}};
\node at (6.5, 2.28) {\tiny Proposition~\ref{res:reparametrization}};
\node[rotate=0] at (3.82, 2.5) {\tiny Proposition~\ref{res:path}};
  \end{tikzpicture}}~~~~~~~~~~~~
  \begin{tikzpicture}
    \draw[-latex] (-0.1, 0) -- (3.5, 0);
    \draw[-latex] (0, -0.1) -- (0, 2);
   \node at (3.3, -0.2) {\scriptsize \parameterGG};
   \node at (0.4, 1.8) {\scriptsize $\lossF{\network_{\parameterGG}}$};
\draw[dotted] (0.6, -0.1) -- (0.6, 1.3);
\draw[dotted] (1, -0.1) -- (1, 1.3);
\draw[dotted] (2, -0.1) -- (2, 0.7);
\draw[dotted] (2.6, -0.1) -- (2.6, 0.7);
\draw[dotted] (3, -0.1) -- (3, 0.7);

   \draw[thick] (0, 1.3)--(0.6, 1.3);
   \draw[thick] (2.6, 0.7)--(3, 0.7);
   \draw[thick, dashed] (0.6, 1.3)--(1, 1.3);
   \draw[thick, dashed] (2, 0.7)--(2.6, 0.7);

   \draw[thick, dashed] plot [smooth] coordinates {(1, 1.3) (1.5, 0.2) (2, 0.7)};

    \filldraw[] (0, 1.3) circle (1pt);
    \filldraw[] (0.6, 1.3) circle (1pt);
    \filldraw[] (1, 1.3) circle (1pt);
    \filldraw[] (2, 0.7) circle (1pt);
    \filldraw[] (2.6, 0.7) circle (1pt);
    \filldraw[] (3, 0.7) circle (1pt);

 \node at (0, -0.4) {\scriptsize ${\parameter}$};

    \node at (1, -0.3646) {\scriptsize $\overline{\parameter}'$};
    \node at (0.6, -0.37) {\scriptsize $\overline{\parameter}$};

    \node at (2, -0.412) {\scriptsize $\underline{\parameterG}'$};

    \node at (2.6, -0.43) {\scriptsize $\underline{\parameterG}$};
    \node at (3, -0.4) {\scriptsize ${\parameterG}$};    

  \end{tikzpicture}
  \caption{path-equivalence between two parameters~$\parameter$ and~$\parameterG$---see Corollary~\ref{res:mainblock} and Definition~\ref{def:path}}
  \label{fig:equivalence}
\end{figure}}
\newcommand{\figureshaped}[1]{
  \begin{figure}[#1]
  \newcommand{\colorboarder}{black}
  \newcommand{\colorfull}{\colorboarder!70!white}
  \newcommand{\colorno}{\colorboarder!10!white}
\newcommand{\scalingfactor}{1}
\newcommand{\spaceinbetween}{\hspace{20mm}}
    \centering
    \scalebox{\scalingfactor}{\begin{tikzpicture}
  \node at (0.46, 0.25) {\scriptsize $\parameterE^2$};
  \fill [\colorno] (0, 0) rectangle (0.8, -0.2);
  \fill [\colorfull] (0, 0) rectangle (0.4, -0.2);
  \draw [\colorboarder] (0, 0) rectangle (0.8, -0.2);

  \node at (1.96, 0.25) {\scriptsize $\parameterE^1$};
  \fill [\colorno] (1, 0) rectangle (2.8, -0.8);
  \fill [\colorfull] (1, 0) rectangle (1.4, -0.4);
  \draw [\colorboarder] (1, 0) rectangle (2.8, -0.8);

  \node at (3.61, 0.25) {\scriptsize $\parameterE^0$};
  \fill [\colorno] (3, 0) rectangle (4.2, -1.8);
  \fill [\colorfull] (3, 0) rectangle (4.2, -0.4);
  \draw [\colorboarder] (3, 0) rectangle (4.2, -1.8);
\end{tikzpicture}}\spaceinbetween\scalebox{\scalingfactor}{\begin{tikzpicture}
  \node at (0.46, 0.25) {\scriptsize $\parameterE^2$};
  \fill [\colorno] (0, 0) rectangle (0.8, -0.2);
  \fill [\colorfull] (0.4, 0) rectangle (0.8, -0.2);
  \draw [\colorboarder] (0, 0) rectangle (0.8, -0.2);

  \node at (1.96, 0.25) {\scriptsize $\parameterE^1$};
  \fill [\colorno] (1, 0) rectangle (2.8, -0.8);
  \fill [\colorfull] (2.4, -0.4) rectangle (2.8, -0.8);
  \draw [\colorboarder] (1, 0) rectangle (2.8, -0.8);

  \node at (3.61, 0.25) {\scriptsize $\parameterE^0$};
  \fill [\colorno] (3, 0) rectangle (4.2, -1.8);
  \fill [\colorfull] (3, -1.4) rectangle (4.2, -1.8);
  \draw [\colorboarder] (3, 0) rectangle (4.2, -1.8);
\end{tikzpicture}}    
    \caption{Illustration of the parameters of an $\shapefactor$-upper block parameter (left) and an $\shapefactor$-lower block parameter (right) for $\nbrlayers=2$. 
The dark areas of the matrices can consist of arbitrary values; 
the light areas consist of zeros.}
    \label{fig:shapes}
  \end{figure}}

\newcommand{\figuretwolayer}[1]{
  \begin{figure}[#1]
  \newcommand{\widthB}{0.8}
  \newcommand{\heightA}{0.6}
  \newcommand{\colorboarder}{black}
  \newcommand{\colorfull}{\colorboarder!70!white}
  \newcommand{\colorno}{\colorboarder!10!white}
  \newcommand{\scalingfactor}{0.8}
  \newcommand{\spaceinbetween}{\hspace{8mm}}
  \newcommand{\emptybox}[2]{\fill[\colorno] (##1, ##2) rectangle (##1+0.2, ##2-0.2);}
  \centering
   \scalebox{\scalingfactor}{\begin{tikzpicture}
     \node at (2.4, 0.8) {Starting point};
     \node at (1.96, 0.25) {\scriptsize $\shmatrixout$};
     \fill [\colorno] (1, 0) rectangle (2.8, -\heightA);

    \fill [\colorfull] (1, 0) rectangle (2.8, -\heightA);
    \draw [\colorboarder] (1, 0) rectangle (2.8, -\heightA);

  \node at (3+\widthB/2, 0.25) {\scriptsize $\shmatrixin$};
  \fill [\colorno] (3, 0) rectangle (3+\widthB, -1.8);

  \fill [\colorfull] (3, 0) rectangle (3+\widthB, -1.8);
  \draw [\colorboarder] (3, 0) rectangle (3+\widthB, -1.8);
\end{tikzpicture}}
  \spaceinbetween
  \scalebox{\scalingfactor}{\begin{tikzpicture}
    \node at (2.4, 0.8) {Step 1};
    \node at (1.96, 0.25) {\scriptsize $\shmatrixoutS$};
    \node at (0.85, -0.25) {\scriptsize $k$};
    \fill [\colorno] (1, 0) rectangle (2.8, -\heightA);
    \fill [\colorfull] (1, 0) rectangle (2.8, -\heightA);
    \emptybox{1}{-0.2}
    \emptybox{1.4}{-0.2}
    \emptybox{1.6}{-0.2}
    \emptybox{1.8}{-0.2}
    \emptybox{2.2}{-0.2}
    \emptybox{2.4}{-0.2}
    \emptybox{2.6}{-0.2}
    \draw [\colorboarder] (1, 0) rectangle (2.8, -\heightA);

  \node at (3+\widthB/2, 0.25) {\scriptsize $\shmatrixin$};
  \fill [\colorno] (3, 0) rectangle (3+\widthB, -1.8);

  \fill [\colorfull] (3, 0) rectangle (3+\widthB, -1.8);
  \draw [\colorboarder] (3, 0) rectangle (3+\widthB, -1.8);
\end{tikzpicture}}
  \spaceinbetween
\scalebox{\scalingfactor}{\begin{tikzpicture}
\node at (2.4, 0.8) {Step 2};
  \node at (1.96, 0.25) {\scriptsize $\widetilde\shmatrixout$};
  \fill [\colorfull] (1, 0) rectangle (2.8, -\heightA);
    \emptybox{1.2}{-0}
    \emptybox{1.4}{-0}
    \emptybox{1.6}{-0}
    \emptybox{1.8}{-0}
    \emptybox{2.0}{-0}
    \emptybox{2.2}{-0}
    \emptybox{2.4}{-0}
    \emptybox{1}{-0.2}
    \emptybox{1.4}{-0.2}
    \emptybox{1.6}{-0.2}
    \emptybox{1.8}{-0.2}
    \emptybox{2.2}{-0.2}
    \emptybox{2.4}{-0.2}
    \emptybox{2.6}{-0.2}
    \emptybox{1.0}{-0.4}
    \emptybox{1.2}{-0.4}
    \emptybox{1.4}{-0.4}
    \emptybox{1.6}{-0.4}
    \emptybox{2.2}{-0.4}
    \emptybox{2.4}{-0.4}
    \emptybox{2.6}{-0.4}
  \draw [\colorboarder] (1, 0) rectangle (2.8, -\heightA);

  \node at (3+\widthB/2, 0.25) {\scriptsize $\shmatrixin$};
  \fill [\colorno] (3, 0) rectangle (3+\widthB, -1.8);

  \fill [\colorfull] (3, 0) rectangle (3+\widthB, -1.8);
  \draw [\colorboarder] (3, 0) rectangle (3+\widthB, -1.8);
\end{tikzpicture}}\spaceinbetween
\scalebox{\scalingfactor}{\begin{tikzpicture}
\node at (2.4, 0.8) {Step 3};
  \node at (1.96, 0.25) {\scriptsize $\shmatrixoutU$};
  \fill [\colorno] (1, 0) rectangle (2.8, -\heightA);
  \fill [\colorfull] (1, 0) rectangle (2.8, -\heightA);
    \emptybox{1.2}{-0}
    \emptybox{1.4}{-0}
    \emptybox{1.6}{-0}
    \emptybox{2.0}{-0}
    \emptybox{2.2}{-0}
    \emptybox{2.4}{-0}
    \emptybox{2.6}{-0}
    \emptybox{1}{-0.2}
    \emptybox{1.4}{-0.2}
    \emptybox{1.8}{-0.2}
    \emptybox{2.0}{-0.2}
    \emptybox{2.2}{-0.2}
    \emptybox{2.4}{-0.2}
    \emptybox{2.6}{-0.2}
    \emptybox{1}{-0.4}
    \emptybox{1.2}{-0.4}
    \emptybox{1.8}{-0.4}
    \emptybox{2.0}{-0.4}
    \emptybox{2.2}{-0.4}
    \emptybox{2.4}{-0.4}
    \emptybox{2.6}{-0.4}
  \draw [\colorboarder] (1, 0) rectangle (2.8, -\heightA);

  \node at (3+\widthB/2, 0.25) {\scriptsize $\shmatrixinU$};
  \fill [\colorno] (3, 0) rectangle (3+\widthB, -1.8);
  \fill [\colorfull] (3, 0) rectangle (3+\widthB, -1.0);
  \draw [\colorboarder] (3, 0) rectangle (3+\widthB, -1.8);
\end{tikzpicture}}
    \caption{overview of the proof of Lemma~\ref{twolayer}}
    \label{fig:twolayer}
  \end{figure}}

\title{Optimization Landscapes of Wide\\Deep Neural Networks Are Benign}

\author{\name Johannes Lederer \email johannes.lederer@rub.de\\
\addr Department of Mathematics,
Ruhr-University Bochum,
www.johanneslederer.com
}

\editor{}

\begin{document}

\maketitle

\begin{abstract}
We analyze the optimization landscapes of deep learning with wide networks.
We highlight the importance of constraints for such networks and show that constraint---as well as unconstraint---empirical-risk minimization over such networks has no \problematics,
that is, suboptimal parameters that are difficult to escape from.
Hence, our theories substantiate the common belief that wide neural networks are not only highly expressive but also  comparably easy to optimize.
\end{abstract}


\section{Introduction}
Deep learning depends on optimization problems that seem impossible to solve,
and yet, deep-learning pipelines outperform their competitors in many applications. 
A common suspicion is that the optimizations are often easier than they appear to be.
In particular,
while most objective functions are nonconvex and, therefore, can have suboptimal local minima,
recent findings suggest that optimizations are not hampered by such suboptimal local minima as long as the neural networks are sufficiently wide.
For example, 
%\citet{Dauphin14} suggest that saddle points, rather than local minima, are the main challenges for optimizations over wide networks;
\citet{Goodfellow14} give empirical evidence for stochastic-gradient descent to converge to a global minimum of the objective function of wide networks;
\citet{Livni2014} show that the optimizations over some classes of wide networks can be reduced to a convex problem;
\citet{Soudry16} suggest that differentiable local minima of objective functions over wide networks are typically global minima;
\citet{Nguyen18} indicate that critical points in  wide networks are often global minima;
and \citet{Allen-Zhu19} and \citet{Du18} suggest that stochastic-gradient descent typically converges to a global minimum for large networks.

These findings raise the question of whether common optimization landscapes over wide (but finite) neural networks have no ``\bads'' altogether.

Progress in this direction has recently been made in \citet{Venturi2019} and then \citet{Lacotte20}.
They characterize the ``\bads'' as parameters from 
which there is no nonincreasing path to a global minimum.
We call these parameters \emph{\problematics} (see Section~\ref{sec:mainresult} for a formal definition).
While the absence of \problematics\ does not preclude saddle points or suboptimal local minima in general,
it means that algorithms are not necessarily trapped in suboptimal points---see Figure~\ref{fig:minima} for an illustration.
\citet{Venturi2019} prove that there are no \problematics\ in some classes of wide networks.
Their theory has two main features that had not been established before:
First, it holds for the entire landscapes---rather than for subsets of them.
This feature is crucial,
because even randomized algorithms typically converge to sets of Lebesgue measure zero with probability one,
which means that  statements about ``almost all'' local minima are not necessarily meaningful.
Second, their theory allows for arbitrary convex loss functions.
This feature is important, for example, in view of the trends toward robust alternatives of the least-squares loss~\citep{Belagiannis15,Jiang18,Wang16}.
On the other hand,
their theory  has three major limitations:
it is restricted to  polynomial activation,
which is convenient mathematically but much less popular than ReLU activation in practice,
it disregards regularizers and constraints,
which have become standard in deep learning and in machine learning at large~\citep{Hastie15} and, as we will show, are essential for wide networks in any case,
and it is restricted to shallow networks, that is, networks with only one hidden layer,
which contrasts the deep architectures that are used in practice~\citep{LeCun2015}.

\figureminima{t}

\citet{Lacotte20} made progress on two of these limitations:
first, 
their theory caters to ReLU activation rather than polynomial activation;
second,
their theory allows for weight decay,
which is a standard way to regularize estimators.
However,
their work is still restricted to shallow,
one-hidden-layer networks.
The interesting question is, therefore, whether such results can also be established for deep networks.
And more generally,
it would be highly desirable to have a theory for the absence of \problematics\ in a broad deep-learning framework.

In this paper,
we establish such a theory.
We prove that the optimization landscapes of empirical-risk minimizers over wide feedforward networks have no \problematics. 
In other words,
we show that every parameter is connected to a global minimum by a nonincreasing path.
Our theory combines the features of the two mentioned works,
as it applies to the entire optimization landscapes,
allows for a wide spectrum of loss functions and activation functions,
and constraint and unconstraint estimation.
Moreover,
it generalizes these works, 
as it allows for multiple outputs and arbitrary depths.
Additionally, our proof techniques are considerably different from the ones used before and, therefore, are  of independent interest.

\paragraph{Guide to the paper}
Sections~\ref{sec:mainpart} and~\ref{sec:discussion} are the basic parts of the paper: 
they contain our main results and the practical implications.
Readers who are interested in the underpinning principles should also study Section~\ref{sec:concepts},
and readers who want additional insights on the proof techniques are referred to~Section~\ref{sec:auxresults}. 
The actual proofs are stated in the Appendix.

\section{Deep-Learning Framework and Main Result}\label{sec:mainpart}
In this section,
we specify the deep-learning framework and state our main results.
The framework includes a wide range of feedforward neural networks;
in particular, 
it allows for arbitrarily many outputs and layers,
a wide range of activation and loss functions,
and constraint as well as unconstraint estimation.
Our main result guarantees that if the networks are sufficiently wide,
the objective function of the empirical-risk minimizer does not have any \problematics.

\subsection{Feedforward Neural Networks}
\label{sec:setup}
We consider input data from a domain $\domainin\subset\R^{\nbrinput}$ and output data from a domain $\domainout\subset\domainoutFull$.
Typical examples are regression data with $\domainout=\domainoutFull$ and classification data with $\domainout=\{\pm 1\}^{\nbroutput}$.
We model the data 
 with  layered, feedforward neural networks, that is,
we study sets of functions 
$\networkS\deq\{\network_{\parameter}\, :\, \domainin\,\to\,\domainoutFull\ :\ \parameter\in\parameterS\}\subset \networkSFull\deq\{\network_{\parameter}\, :\, \domainin\,\to\,\domainoutFull\ :\ \parameter\in\parameterSFull\}$
with
\begin{equation}\label{networks}
  \network_{\parameter}[\vectorin]\deq\parameterl\activation^{\nbrlayers}\bigl[\parameterE^{\nbrlayers-1}\cdots \activation^1[\parameterE^0\vectorin]\bigr]~~~~~~~~~\text{for}~\vectorin\in\domainin
\end{equation}
and 
\begin{equation*}
  \parameterS\subset\parameterSFull\deq\bigl\{\parameter=(\parameterl,\dots,\parameterE^0)\ :\ \parameterE^j\in\R^{\nbrparameterjj\times\nbrparameterj}\bigr\}\,.
\end{equation*}
The quantities 
 $\nbrparameter^0=\nbrinput$ and~$\nbrparameter^{\nbrlayers+1}=\nbroutput$ are the input and output dimensions, respectively,
 $\nbrlayers$~the depth of the networks,
and $\nbrwidthmin\deq \min\{\nbrparameter^1,\dots,\nbrparameterl\}$ the minimal  width of the networks.
The functions $\activation^j\,:\,\R^{\nbrparameterj}\to\R^{\nbrparameterj}$ are called the activation functions.
\label{activationfunctions}We assume that the activation functions are elementwise functions in the sense that
$\activation^j[\boldsymbol{b}]=(\activationindjF{b_1},\dots,\activationindjF{b_{\nbrparameterj}})\tp$ for all $\boldsymbol{b}\in\R^{\nbrparameterj}$,
where $\activationindj\,:\,\R\to\R$ is an arbitrary function.
% is an nonnegative homogeneous function of degree~$k$ for some $k\in(0,\infty)$---see \citet[Page~2]{Taheri20}.
This allows for an unlimited variety in the type of activation,
including ReLU $\activationindj\,:\,b\mapsto \max\{0,b\}$,
leaky ReLU $\activationindj\,:\,b\mapsto \max\{0,b\}+\min\{0, cb\}$ for a fixed $c\in(0,1)$,
polynomial $\activationindj\,:\,b\mapsto cb^{\activationdegree}$  
 for  fixed $c\in(0,\infty)$ and $\activationdegree\in[1,\infty)$,
and sigmoid activation  $\activationindj\,:\,b\mapsto 1/(1+e^{-b})$ as popular examples,
and it allows for different activation functions in each layer.

We study the most common approaches to parameter estimation in this setting:
constraint and unconstraint empirical-risk minimization.
The loss function $\loss\,:\,\domainoutFull\times \domainoutFull\to\R$ is assumed convex in its first argument;
this includes all standard loss functions,
such as the least-squares loss $\loss\,:\,(\boldsymbol{a},\boldsymbol{b})\mapsto \normtwos{\boldsymbol{a}-\boldsymbol{b}}$,
the absolut-deviation loss $\loss\,:\,(\boldsymbol{a},\boldsymbol{b})\mapsto \normone{\boldsymbol{a}-\boldsymbol{b}}$ (both typically used for regression), 
the logistic loss $\loss\,:\,(a,b)\mapsto -(1+b)\log[1+a]-(1-b)\log[1-a]$,
the hinge loss $\loss\,:\,(a,b)\mapsto \max\{0,1-ab\}$
(both typically used for binary classification $\domainout=\{\pm1\}$),
and so forth.
The optimization domain is the set
\begin{equation*}
  \parameterS\deq\bigl\{\parameter\in\parameterSFull\ :\ \regularizerF{\parameter}\leq 1\bigr\}
\end{equation*}
for  a constraint $\regularizer\,:\,\parameterSFull\to\R$.
Given data $(\vectorin_1,\vectorout_1),\dots, (\vectorin_{\nbrsamples},\vectorout_{\nbrsamples})\in\domainin\times \domainout$,
the empirical-risk minimizers are then  the networks $\network_{\estimator_{\operatorname{erm}}}$ with 
\begin{equation}\label{erm}
    \estimator_{\operatorname{erm}}\in\argmin_{\parameter\in\parameterS}\Biggl\{\sum_{i=1}^{\nbrsamples}\lossFB{\networkAF{\vectorini},\vectorouti}\Biggr\}\,.
\end{equation}
Another formulation of such estimators is based on the Lagrange form 
\begin{equation*}
  \argmin_{\parameter\in\parameterSFull}\Biggl\{\sum_{i=1}^{\nbrsamples}\lossFB{\networkAF{\vectorini},\vectorouti}+\regularizerF{\parameter}\Biggr\}\,,
\end{equation*}
which replaces the constraint by a regularization term.
Imposing constraints and adding regularization terms are really two sides of the same coin,
so that we can focus on constraints for simplicity.

It has been shown that regularizers/constraints can facilitate the optimization as well as improve generalization---see \citet{Krizhevsky2012} and \citet{Livni2014}, among others.
For ease of presentation,
we limit ourselves to the following class of constraints:
\begin{equation}\label{def:constraint}
  \regularizerF{\parameter}\deq\max\Bigl\{\tuningparameterout\max_{j\in\{1,\dots,\nbrlayers\}}\normoneM{\parameterj},\tuningparameterin\normqM{\parameterz}\Bigr\}~~~~~~~~~\text{for all}~\parameter\in\parameterSFull
\end{equation}
for  fixed tuning parameters $\tuningparameterout,\tuningparameterin\in[0,\infty)$, 
a parameter~$\losspower\in(0,\infty]$,
and $\normqM{\cdot}$ the usual row-wise $\ell_q$``-norm,''
that is,
$\normqM{\parameterj}\deq \max_k(\sum_{i}\abs{(\parameterj)_{ki}}^q)^{1/q}$ for $q\in(0,\infty)$ and $\normsupM{\parameterj}\deq \max_{ki}\abs{(\parameterj)_{ki}}$.
This class of constraints includes the following four important cases:
\begin{itemize}\label{def:cases}
\item \textit{Unconstraint estimation:} $\tuningparameterout=\tuningparameterin=0$.\vspace{-2mm}
\end{itemize}
In other words, $\parameterS=\parameterSFull$.
Unconstraint estimation had been the predominant approach in the earlier days of deep learning  \citep{Anthony09}.
\begin{itemize}
\item \textit{Connection sparsity:} $\losspower=1$.\vspace{-2mm}
\end{itemize}
This constraint yields connection-sparse networks,
which have received considerable attention  recently~\citep{Barron2018,Barron2019,Kim16,Taheri20}.
\begin{itemize}
\item \textit{Strong sparsity:} $\losspower<1$.\vspace{-2mm}
\end{itemize}
Nonconvex constraints have been popular in statistics for many years \citep{Fan01,Zhang10}, 
but our paper is probably the first one to include such constraints in a theoretical analysis in deep learning.
\begin{itemize}
\item \textit{Input constraints:} $\tuningparameterout=0$.\vspace{-2mm}
\end{itemize}
Some researchers have argued for applying certain constraints,
such as node-sparsity, only to the input level~\citep{Feng17}.
In general, while our proof techniques also apply to many other types of constraints,
there are two main reasons for using the mentioned sparsity-inducing constraints to illustrate our results:
First, sparsity has become very popular in deep learning, 
because it can lower the burden on memory and optimization as well as increase interpretability  \citep{Hebiri20}.
And second,
the above examples allow us to demonstrate that the discussed features of wide networks do not depend on smooth and convex constraints  such as weight decay.

% Our theory can also be adjusted to the regularized versions of the empirical-risk minimizers,
% that is,
% for the networks indexed by any parameter in the set
% \begin{equation*}
% \argmin_{\parameter\in\parameterSFull}\Biggl\{\sum_{i=1}^{\nbrsamples}\lossFB{\networkAF{\vectorini},\vectorouti}+\regularizerF{\parameter}\Biggr\}\,.
% \end{equation*}
% The proofs are virtually the same as for the constraint versions;
% we omit the details for the sake of brevity.

One line of research develops statistical theories for constraint and unconstraint empirical-risk minimizers---see \citet{Bartlett02} and \citet{Lederer20b}, among others.
As detailed above,
empirical-risk minimizers are the networks whose parameters are global minima of the objective function
\begin{equation}\label{def:objective}
\parameter\ \mapsto\ \lossF{\networkA}\deq\sum_{i=1}^{\nbrsamples}\lossFB{\networkAF{\vectorini},\vectorouti}
\end{equation}
over \parameterS\ for fixed data $(\vectorin_1,\vectorout_1),\dots, (\vectorin_{\nbrsamples},\vectorout_{\nbrsamples})$.
While the function $\networkA\mapsto\lossF{\networkA}$ is convex by assumption,
the objective function $\parameter\ \mapsto\ \lossF{\networkA}$ is usually nonconvex.
It is thus unclear, per se, whether deep-learning pipelines can be expected to yield global minima of the objective function and, therefore,
whether the statistical theories are valid in practice. 
Our goal is, broadly speaking, to establish conditions under which global minimization of~\eqref{def:objective} can indeed be expected.

\subsection{Absence of \PRoblematics}\label{sec:mainresult}
Objective functions in deep learning usually have many suboptimal local minima and saddle points (see \citep{Petzka2018} and references therein). 
Since the slopes (or appropriate generalizations of them) at local minima and saddle points are zero, 
 basic gradient-descent algorithms can get stuck in these points.
 But contemporary optimization schemes are based on stochastic-gradient descent rather than basic gradient descent and often use momentum and other additional schemes, 
so that a zero-valued gradient alone is not worrisome (see \citet{Goldblum2019}, for example).
This means that we have to characterize the points that are most problematic for global optimization more carefully.
We do this by introducing the notion of \emph{\problematics} and show that such points do not exist in sufficiently wide networks.

\label{def:minima}Recall first that a parameter $\parameter\in\parameterS$ that satisfies 
\begin{equation*}
  \lossF{\networkA}\leq \lossF{\networkAG}~~~~~~~~~\text{for all}~\parameterG\in\parameterS~\text{with}~\normM{\parameter-\parameterG}\leq c
\end{equation*}
for a value $c\in(0,\infty)$ and a norm~$\normM{\cdot}$ on~\parameterSFull\ is called a \emph{local minimum} of the objective function~\eqref{def:objective}. 
If the statement holds for every $c\in(0,\infty)$,
the parameter $\parameter$ is called a \emph{global minimum}.
We then formalize the notion of  \problematics\ as follows:
\begin{definition}[\Problematics]\label{def:spurious}
 Consider a parameter $\parameter\in\parameterS$.
If there is no continuous function $\pathfunction\,:\,[0,1]\to\parameterS$ that satisfies (i)~$\pathfunctionF{0}=\parameter$ and $\pathfunctionF{1}=\parameterG$ for a global minimum~$\parameterG\in\parameterS$ of the objective function~\eqref{def:objective} and (ii)~$\paradd\mapsto\lossF{\network_{\pathfunctionF{\paradd}}}$ 
 is nonincreasing,
we call the parameter~$\parameter$ a \emph{\problematic}.
\end{definition}
The notion of \problematics\  aligns with previous notions 
such as ``setwise strict local minima'' 
\citep[Definition~1]{Li2018}, 
``non-attracting regions of local minima'' \citep[Definition~1]{Petzka2018},
``bad local valleys'' \citep[Definition~4.1]{Nguyen2019},
and ``spurious valleys'' \citep[Definition~1]{Venturi2019}; \citep[Definition~1]{Lacotte20}.
The two main reasons for giving another formal definition are 1.~the convenience of  a pointwise (rather than setwise) notion and 
2.~the ambiguity  of the term ``spurious'' in the literature.

In brief, \problematics\ are suboptimal parameters that are difficult to escape from---see Figure~\ref{fig:minima}.
\Problematics\ can be local minima,
but in constraint estimation, for example,
there can  be boundary effects that also make other parameters \problematics.
In any case,
\problematics\ are typically the most problematic parameters for a global optimization of the objective function.
In contrast,
saddle points and plateaus that ``have an exit'' (see again Figure~\ref{fig:minima}) are not characterized as \problematics,
because they can be escaped from---at least in principle---by using momentum or other schemes \citep{Dauphin14,Goldblum2019}.

Our main technical result in this paper is that the objective function~\eqref{def:objective} has no \problematics\ if the networks are sufficiently wide.
\begin{theorem}[Absence of \problematics]\label{mainresult}
Consider the setup of Section~\ref{sec:setup}.
  If $\nbrwidthmin\geq 2\nbroutput(\nbrsamples+1)^{\nbrlayers}$,
the objective function~\eqref{def:objective} has no \problematics.
\end{theorem}
\noindent In other words,
empirical-risk minimization over sufficiently wide networks does not involve ``problematic'' parameters.
Hence, as long as there are means to circumvent potentially difficult but ``\unproblematic'' parameters such as saddle points \citep{Dauphin14},
it is reasonable to expect global optimization in wide networks.
More broadly speaking, the main message of Theorem~\ref{mainresult} is that 
 optimizations become easier with increasing widths.

The theorem applies very generally.
First, it includes all \bads\ rather than ``many'' or ``almost all'' \bads.
This feature is important,
because even randomized algorithms usually converge to a few, fixed points with high probability.
Second,
the framework allows for arbitrary convex loss functions.
This feature caters, for example, to a current trend toward robust alternatives of the least-squares loss function~\citep{Barron19,Lederer20b}.
Third,
the framework includes ReLU activation.
ReLU activation is nondifferentiable and, therefore, mathematically  more  challenging than, for example, linear and polynomial activation, 
but it has become the predominant type of activation in practice.
Forth,
the framework includes constraint as well as unconstraint estimation.
Constraint estimation is particularly suitable for wide networks, and for overparameterized networks more generally,
because it can avoid overfitting and facilitate optimizations;
moreover, as we see below,
unconstraint networks are questionable in any case.
Fifth,
our statement holds for arbitrary output dimensions and depths.
The latter is particularly important in view of the current trend toward deep architectures.
In sum, our result is  sweeping proof for the fact that wide networks have no \problematics,
and it sheds light on the  optimization landscapes of deep learning more generally.

Despite its generality,
the theorem specializes correctly.
The bound on the network widths becomes $2\nbroutput(\nbrsamples+1)^{\nbrlayers}=2(\nbrsamples+1)$ in the   case of a single output~($\nbroutput=1$) and a single hidden layer ($\nbrlayers=1$),
which coincides with the bounds that have been established for shallow networks with one output and specific activation functions and estimators---see \citet{Lacotte20} and references therein.

In addition,
we can remove the dependence on~\nbrlayers\ completely in the case of linear activation:
\begin{theorem}[Linear activation]\label{mainresultlinear}
Consider the setup of Section~\ref{sec:setup} with linear activation.
  If $\nbrwidthmin\geq 2\nbroutput(\nbrsamples+1)$,
the objective function~\eqref{def:objective} has no \problematics.
\end{theorem}
\noindent Hence, the lower bound $2\nbroutput(\nbrsamples+1)^{\nbrlayers}$ is reduced to $2\nbroutput(\nbrsamples+1)$ in the case of linear activation $\activationindj\,:\,b\mapsto b$. 
Linear-activation networks are pointless from a statistical perspective,
but insights into their optimization are still valuable,
since the optimization landscapes of these networks  are surprisingly involved---potentially as involved as the landscapes of any other type of network.

The fact that the rates coincide with known rates  in the case $\nbroutput=\nbrlayers=1$ indicates that our theory takes account of the sample size~$\nbrsamples$  somewhat optimally.
In contrast, the question of whether the rate's dependence on $\nbroutput$ and~\nbrlayers\ is optimal is completely open,
as there are no corresponding results in the literature---neither empirical nor theoretical results.
Indeed,
 (i)~empirical studies usually concern generalization errors,
which includes many confounding aspects, such as expressivity and overfitting.

And (ii)~theories that indicate no dependence  on~\nbrlayers, 
such as \citet{Li2018,Nguyen2019} and many others, 
restrict the activation functions, structure of the network, estimator, or other aspects.
Most importantly,
they omit constraints and consider only ``almost all'' parameters.
The  inclusion of constraints (or similarly regularization) makes our proofs much more  interesting and challenging from a technical perspective and is essential from a practical perspective.
On the technical side,
observe that  approaches that use basis functions do no longer apply in the presence of constraints (cf.~\citet{Li2018,Venturi2019}, for example);
the technical reasons are boundary effects and the lack of necessary symmetries.
Besides,
our approach provides a more intuitive understanding of the forces at play than approaches based on basis functions---see our Section~\ref{sec:concepts}---and avoids any assumptions on the data---cf.~\citet[Assumption~1.1]{Li2018}, for example.

On the practical side,
 \problematics\ and similar concepts are questionable in wide networks without constraints,
and unconstraint wide networks are essentially meaningless altogether.
We support this claim with two arguments:
First, in the absence of constraints,
a point could be connected only to global optima that are extremely far from the origin;
however, such global optima are likely useless due to overfitting,
which means that the point is unproblematic from a pure optimization perspective but still problematic more generally. 
Second, we can repurpose the elegant convergence arguments of \citet{Li2018} to derive the following insight:
\begin{theorem}[Unconstraint networks]\label{unconstraint}
Consider the setup of Section~\ref{sec:setup} without constraints, that is,
$\parameterS=\parameterSFull$.
  If $\nbrwidthmin\geq 2\nbroutput(\nbrsamples+1)$, it holds for Lebesgue-almost every $\parameterE^{\nbrlayers-1}\in\R^{\nbrparameterl\times\nbrparameter^{\nbrlayers-1}},\dots,\parameterE^0\in\R^{\nbrparameter^1\times\nbrparameter^0}$ that 
there is a $\parameterE^{\nbrlayers}$ such that $\parameter=(\parameterE^{\nbrlayers},\dots,\parameterE^0)\in\parameterS$ is a global minimum of the objective function~\eqref{def:objective}.
\end{theorem}
\noindent This means that we can arbitrarily fix the inner layers and just optimize the outer layer,
which is---of course---just a linear regression on arbitrarily transformed data and, therefore, has little to do with neural networking and deep learning.

Most previous papers on the topic do not include constraints.
Notable exceptions are \citet{Haeffele2017} and \citet{Lacotte20};
however, the results in those two papers are limited (or essentially limited---see \citet[Sections~5 and~6]{Haeffele2017}) to shallow networks, that is, $\nbrlayers=1$,
and make strict assumptions on the network architecture (such as ReLU activation).
We could, therefore, think of Theorem~\ref{mainresult} as a generalization of their results.

In addition to the lack of constraints,
most existing results in our context are restricted to Lebesgue-almost all parameters---see \citet{Nguyen18} and many others.
But this restriction can make the results almost meaningless---see the illustration in Figure~\ref{fig:almostall}.
Hence, our inclusion of the entire landscape is of high practical relevance.

\newcommand{\figurealmostall}[1]{\begin{figure}[#1]
  \centering
  \begin{tikzpicture}[scale=1, domain=-0.1:4]
    % Axes 
    \draw[-latex] (-0.1, 0) -- (4, 0);
    \draw[-latex] (0, -0.1) -- (0, 2);
    % Axes labels
   \node at (3.8, -0.2) {\scriptsize \parameter};
   \node at (0.4, 1.9) {\scriptsize $\lossF{\networkA}$};

  \draw[color=black, thick, smooth, samples=300]   plot (\x, {0.25*(\x-2)*(\x-2) + 0.4*sin(10*\x r) + 0.5 });
  \end{tikzpicture}
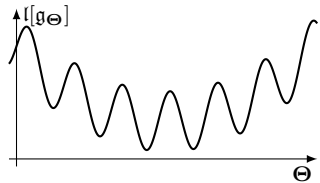
  \caption{a hypothetical objective function that has Lebesgue-almost no \problematics\ but is very hard to optimize}
  \label{fig:almostall}
\end{figure}}

\figurealmostall{t}

\section{Underlying Concepts}\label{sec:concepts}
In this section, we introduce concepts that we use in our proofs and that might also be of interest more generally.
We first formulate the notion of path equivalence,
which yields a practical characterization of \problematics.
We then devise specific parameters that can act as mediators between path-equivalent parameters.
The main reason why our proof techniques are quite different from what can be found in the literature is that we cater to deep networks and a range of  activation functions.

\subsection{Path Relations}
\label{sec:pe}
The objective functions for optimizing neural networks  are typically continuous but not convex or differentiable.
In the following,
we characterize the absence of \problematic\ in a way that suits these characteristics of neural networks. 
The key concept is formulated in the following definition.
\begin{definition}[Path relations]\label{def:path}
Consider two parameters $\parameter,\parameterG\in\parameterS$.
If there is a continuous function $\pathfunction_{\parameter,\parameterG}\,:\,[0,1]\to\parameterS$ that satisfies  
$\pathfunction_{\parameter,\parameterG}[0]=\parameter$,
$\pathfunction_{\parameter,\parameterG}[1]=\parameterG$,
and $\paradd\mapsto \lossF{\network_{\pathfunction_{\parameter,\parameterG}[\paradd]}}$ is constant,
we say that $\parameter$ and $\parameterG$ are \emph{path constant} and write $\parameter\pdec\parameterG$.

If there is a continuous function $\pathfunction_{\parameter,\parameterG}\,:\,[0,1]\to\parameterS$ that satisfies  
$\pathfunction_{\parameter,\parameterG}[0]=\parameter$,
$\pathfunction_{\parameter,\parameterG}[1]=\parameterG$,
and $\paradd\mapsto \lossF{\network_{\pathfunction_{\parameter,\parameterG}[\paradd]}}$ is convex,
we say that $\parameter$ and $\parameterG$ are \emph{path convex} and write $\parameter\pcon\parameterG$.

If there are parameters  $\parameter',\parameterG'\in\parameterS$ such that (i)~$\parameter\pdec\parameter'$ and $\parameterG\pdec\parameterG'$ and (ii)~$\parameter'\pcon\parameterG'$, 
we say that $\parameter$ and $\parameterG$ are \emph{path equivalent} and write $\parameter\pequi\parameterG$.
\end{definition}
Path constantness means that two parameters are connected by a continuous path of parameters that is constant with respect to the loss;
path convexity relaxes ``constant'' to ``convex;''
path equivalence allows for additional mediators.
The three relations are ordered in the sense that 
$\parameter\pdec\parameterG\Rightarrow\parameter\pcon\parameterG\Rightarrow \parameter\pequi\parameterG$,
and they satisfy a number of other basic properties.
\begin{lemma}[Basic properties]\label{res:basic}
It holds for all $\parameter,\parameterG,\parameterGG\in\parameterS$ that
  \begin{enumerate}
  \item $\parameter\pdec\parameter$; $\parameter\pcon\parameter$;  and  $\parameter\pequi\parameter$ (reflexivity);
  \item $\parameter\pdec\parameterG\Rightarrow\parameterG\pdec\parameter$; $\parameter\pcon\parameterG\Rightarrow\parameterG\pcon\parameter$; and $\parameter\pequi\parameterG\Rightarrow\parameterG\pequi\parameter$ (symmetry).
  \item $\parameter\pdec\parameterG$ and $\parameterG\pdec\parameterGG\Rightarrow\parameter\pdec\parameterGG$ (transitivity).
  \end{enumerate}
\end{lemma}
\noindent The proof  is straightforward and, therefore, omitted.
The lemma illustrates that the path relations  equip the parameter space with solid mathematical structures.

We can finally use the above-stated concepts to characterize \problematics.
\begin{proposition}[Characterization of \problematics]\label{res:characterization} 
Assume that  for all $\parameter\in\parameterS$,
there is a global minimum of the objective function~\eqref{def:objective}, denoted by~$\parameterG$, such that $\parameter\pequi\parameterG$.
Then, the objective function~\eqref{def:objective} has  no \problematics.
\end{proposition}
\noindent 
Hence, path equivalence of all parameters to a global minimum   is a sufficient condition for the absence of \problematics.
This statement is the main result of Section~\ref{sec:pe}.

\subsection{Block Parameters}
\label{sec:bl}
\figureshaped{t}

The parameterization of neural networks is typically ambiguous:
many different parameters yield the same network.
We  leverage this ambiguity to make the networks more tractable.
The key concept is formulated in the following definition.

\begin{definition}[Block parameters]\label{def:blocks}
Consider a number $\shapefactor\in\{0,1,\dots\}$ and a parameter $\parameter\in\parameterS$.
If
  \begin{enumerate}
  \item $(\parameterE^0)_{ji}=0$ for all $j>\shapefactor$;
  \item $(\parameterE^v)_{ij}=0$ for all $v\in\{1,\dots,\nbrlayers-1\}$ and $i>\shapefactor$ and for all $v\in\{1,\dots,\nbrlayers-1\}$ and $j>\shapefactor$;
  \item $(\parameterl)_{ij}=0$ for all $j>\shapefactor$,
  \end{enumerate}
we call~\parameter\ an \emph{\shapefactor-upper-block parameter of depth~\nbrlayers}.

  Similarly, if
  \begin{enumerate}
  \item $(\parameterE^0)_{ji}=0$ for all $j\leq \nbrparameter^1-\shapefactor$;
  \item $(\parameterE^v)_{ij}=0$ for all $v\in\{1,\dots,\nbrlayers-1\}$ and $i\leq \nbrparameter^{v+1}-\shapefactor$ and for all $v\in\{1,\dots,\nbrlayers-1\}$ and $j\leq \nbrparameter^{v}-\shapefactor$;
  \item $(\parameterl)_{ij}=0$ for all $j\leq \nbrparameterl-\shapefactor$,
  \end{enumerate}
we call~\parameter\ an \emph{\shapefactor-lower-block parameter  of depth~\nbrlayers}.
We denote the sets of the $\shapefactor$-upper-block and $\shapefactor$-lower-block parameters  of depth~\nbrlayers\ by \shapeupper\ and \shapelower, respectively.
\end{definition}
\noindent 
Trivial examples are the $0$-block parameters $\shapeupperE_0=\shapelowerE_0=\{\zero=(\zero_{\nbrparameterll\times\nbrparameterl},\dots,\zero_{\nbrparameter^1\times\nbrparameter^0})\}$
and  the $\shapefactor$-block parameters $\shapeupper=\shapelower=\parameterS$ for $\shapefactor\geq\max\{\nbrparameter^1,\dots,\nbrparameterl\}$.
More generally,
the block parameters consist of block matrices:
see Figure~\ref{fig:shapes}.
We show in the following that block parameters can be mediators in the sense of path equivalence.

We first show that every parameter is path constant to a block parameter.
\begin{proposition}[Path connections to block parameters]\label{res:reparametrization}
For every $\parameter\in\parameterS$ and $\shapefactor\deq\nbroutput(\nbrsamples+1)^{\nbrlayers}$,
there are $\parameterU,\parameterL\in\parameterS$ with $\parameterU\in\shapeupper$ and
$\parameterL\in\shapelower$
such that $\parameter\pdec\parameterU$ and $\parameter\pdec\parameterL$.
\end{proposition}
\noindent In particular,
every parameter is path connected to both an upper-block parameter and a lower-block parameter.
The interesting cases are wide networks:
for fixed~\shapefactor,
the wider the network,
the more pronounced the block structure.

We then show that there is a connection between upper-block and lower-block parameters.
\begin{proposition}[Path connections among block parameters]\label{res:path} 
Consider two block parameters $\parameter\in\shapeupper$ and $\parameterG\in\shapelower$.
If  $\nbrwidthmin\geq 2\shapefactor$,
it holds that $\parameter\pequi\parameterG$.
\end{proposition}
\noindent Hence,
every upper-block parameter is path connected to every lower-block parameter---as long as the minimal width of the networks is sufficiently large.

We finally combine Propositions~\ref{res:reparametrization} and~\ref{res:path}.
\begin{corollary}[All parameters are path equivalent]\label{res:mainblock}
Consider two arbitrary parameters  $\parameter,\parameterG\in\parameterS$.
If $\nbrwidthmin\geq 2\nbroutput(\nbrsamples+1)^{\nbrlayers}$,
it holds that $\parameter\pequi\parameterG$.  
\end{corollary}
See Figure~\ref{fig:equivalence} for an illustration. 
The corollary ensures that  as long as the minimal width is sufficiently large, all networks are path equivalent.
This result, therefore, connects directly to the characterization of \problematics\ in Proposition~\ref{res:characterization} of the previous section.

\figureshaping{t}

\section{Auxilliary Results}\label{sec:auxresults}
In this section, 
we state four auxiliary results.
% To simplify the notation, 
% we assume throughout the remainder of this paper that the activation functions are nonnegative-homogeneous of degree~$\activationdegree=1$;
% it is straightforward to adjust the proofs for general~$\activationdegree\in(0,\infty)$.

\subsection{Two-Layer  Networks}
Here, we show that two-layer networks can be reparametrized such that they are indexed by block parameters.
We first introduce the notation
\begin{equation*}\label{def:regularizerbasic}
  \regularizerRF{\matrixG,\losspower}\deq\normqM{\matrixG}=\max_{a\in\{1,\dots,b\}}\Biggl(\sum_{j=1}^{c}\abs{\matrixG_{aj}}^{\losspower}\Biggr)^{1/\losspower}~~~~~~~~~~~~~~\text{for all}~\matrixG\in\R^{b\times c},\,\losspower\in(0,\infty)
\end{equation*}
and $\regularizerRF{\matrixG,\infty}\deq \normsupM{\matrixG}=\max_{aj}\abs{\matrixG_{aj}}$ for all $\matrixG\in\R^{b\times c}$.
These functions are the  building blocks of the constraint on Page~\pageref{def:constraint}.
Next, given a  permutation $\perm\,:\,\{1,\dots,c\}\to\{1,\dots,c\}$ and a matrix $\matrixG\in\R^{b\times c}$,
we define the matrix $\matrixG_{\perm}\in\R^{b\times c}$ through $(\matrixG_{\perm})_{ij}\deq \matrixG_{i\permF{j}}$.
Similarly, given a  permutation $\perm\,:\,\{1,\dots,b\}\to\{1,\dots,b\}$ and a matrix $\matrixG\in\R^{b\times c}$,
we define the matrix $\matrixG^{\perm}\in\R^{b\times c}$ through $(\matrixG^{\perm})_{ji}\deq \matrixG_{\permF{j}i}$.
The result is then the following:
\begin{lemma}[Two-Layer networks]\label{twolayer}
Consider three matrices $\shmatrixout\in\R^{\tlnbrout\times \tlnbrmed}$, 
$\shmatrixin\in\R^{\tlnbrmed\times \tlnbrin}$,
and $\shmatrixdata\in\R^{\tlnbrin\times \tlnbrdata}$,
two constants $\losspowerout\in(0,1]$ and $\losspowerin\in(0,\infty]$,
and a function  $\shfunction\,:\,\R\to\R$.
With some abuse of notation,
define 
$\shfunction\,:\,\R^{\tlnbrmed \times \tlnbrdata}\to\R^{\tlnbrmed\times \tlnbrdata}$ through $(\shfunctionF{\matrixG})_{ji}\deq \shfunctionF{\matrixG_{ji}}$ for all $\matrixG\in\R^{\tlnbrmed\times \tlnbrdata}$.
Then, there are matrices $\shmatrixoutU\in\R^{\tlnbrout\times \tlnbrmed}$ and $\shmatrixinU\in\R^{\tlnbrmed \times \tlnbrin}$ and a permutation $\perm\,:\,\{1,\dots,\tlnbrmed\}\to\{1,\dots,\tlnbrmed\}$ such that
\begin{itemize}
\item $\shmatrixoutU\shfunctionF{\shmatrixinU\shmatrixdata}=\shmatrixout_{\perm}\shfunctionF{\shmatrixin^{\perm}\shmatrixdata}$;
\item $\regularizerRF{\shmatrixoutU,\losspowerout}\leq \regularizerRF{\shmatrixout_{\perm},\losspowerout}$ and $\regularizerRF{\shmatrixinU,\losspowerin}\leq \regularizerRF{\shmatrixin^{\perm},\losspowerin}$;
\item $\shmatrixoutU_{ij}=0$ for $j>\tlnbrout(\tlnbrdata+1)$; $\shmatrixinU_{ji}=0$ for $j>\tlnbrout(\tlnbrdata+1)$ and $\shmatrixinU_{ji}=(\shmatrixin^{\perm})_{ji}$ otherwise.
\end{itemize}
Similarly,
there are matrices $\shmatrixoutL\in\R^{\tlnbrout\times \tlnbrmed}$ and $\shmatrixinL\in\R^{\tlnbrmed\times \tlnbrin}$ and a permutation $\perm\,:\,\{1,\dots,\tlnbrmed\}\to\{1,\dots,\tlnbrmed\}$ such that
\begin{itemize}
\item $\shmatrixoutL\shfunctionF{\shmatrixinL\shmatrixdata}=\shmatrixout_{\perm}\shfunctionF{\shmatrixin^{\perm}\shmatrixdata}$;
\item $\regularizerRF{\shmatrixoutL,\losspowerout}\leq \regularizerRF{\shmatrixout_{\perm},\losspowerout}$ and $\regularizerRF{\shmatrixinL,\losspowerin}\leq \regularizerRF{\shmatrixin^{\perm},\losspowerin}$;
\item $\shmatrixoutL_{ij}=0$ for $j\leq \tlnbrmed-\tlnbrout(\tlnbrdata+1)$; $\shmatrixinU_{ji}=0$ for $j\leq \tlnbrmed-\tlnbrout(\tlnbrdata+1)$ and $\shmatrixinU_{ji}=(\shmatrixin^{\perm})_{ji}$ otherwise.
\end{itemize}
\end{lemma}
\noindent Hence,
the parameter matrices of two-layer networks can be brought into the shapes illustrated in Figure~\ref{fig:shapes}.
We apply this result repeatedly in the proof of Proposition~\ref{res:reparametrization}.

\subsection{Symmetry Property of Neural Networks}
Next, we point out a symmetry in our setup for the neural networks.
\begin{lemma}[Symmetry property]\label{res:symmetry}
Consider permutations $\permj\,:\,\{1,\dots,\nbrparameterj\}\to\{1,\dots,\nbrparameterj\}$ for $j\in\{0,\dots,\nbrlayers+1\}$.
Assume that $\perm^{0}$ and $\perm^{\nbrlayers+1}$ are the identity functions: $\perm^{0}[j]=\perm^{\nbrlayers+1}[j]=j$ for all~$j$.
The parameter $\parameter\in\parameterS$ is a \problematic\  of the objective function~\eqref{def:objective}
if and only if  $\parameterG\in\parameterS$ defined through $(\parameterGj)_{uv}\deq(\parameterj)_{\permjjF{u}\permjF{v}}$ for all $j\in\{0,\dots,\nbrlayers\}$, $u\in\{1,\dots,\nbrparameterjj\}$, and $v\in\{1,\dots,\nbrparameterj\}$
% \begin{equation*}
%   (\parameterGj)_{uv}\deq(\parameterj)_{\permjjF{u}\permjF{v}}~~~~~~~~~~~\text{for all}~j\in\{0,\dots,\nbrlayers\},\,u\in\{1,\dots,\nbrparameterjj\},\,v\in\{1,\dots,\nbrparameterj\}
% \end{equation*}
is a \problematic\  of the objective function~\eqref{def:objective}. 
\end{lemma}
\noindent The proof follows readily from our setup in Section~\ref{sec:setup} and, therefore, is omitted. 
The lemma illustrates that the parameterizations of neural networks are highly ambiguous. 
But in this case, the ambiguity is convenient,
because it allows us to permute the rows and columns of the parameters to bring the parameters in shapes that are easy to manage.

\subsection{Property of Convex Functions}
We now establish a simple property of  convex  functions.

\begin{lemma}[Property of convex functions]\label{res:convex}
Consider a convex function~$\pathfunction\,:\,[0,1]\to\R$.
If $\pathfunctionF{0}>\pathfunctionF{\paraddP}$ for a $\paraddP\in(0,1]$,
there is a $\paraddPP\in\argmin_{\paradd\in(0,1]}\{\pathfunctionF{\paradd}\}$ such that the function $\pathfunctionP\,:\,[0,1]\to\R$ defined through $\pathfunctionPF{\paradd}\deq \pathfunctionF{\paraddPP\paradd}$ for all $\paradd\in[0,1]$ is nonincreasing and $\pathfunctionPF{0}>\pathfunctionPF{1}$.
\end{lemma}
\noindent 
This lemma connects the convexity from Definition~\ref{def:path} with the \problematics\  from Definition~\ref{def:spurious}.
We use this result in the proof of Proposition~\ref{res:characterization}.

\subsection{Carath\'eodory-Type Result}
Carath\'eodory's theorem goes back to \citet{Caratheodory11};
see \citet{Boltyanski01,Fenchel29,Hanner51} for related results.
The following statement combines the classical theorem and the much more recent  results \citet[Theorem~1 and Lemma~1]{Bastero95}.
\begin{lemma}[Carath\'eodory-Type result]\label{res:caraold}
  Consider a number $\losspower\in(0,1]$, vectors  $\caravector_1,\dots,\caravector_h\subset\R^{\tlnbrdata}$, 
and the vectors' $\losspower$-convex hull   $ \operatorname{conv}_{\losspower}[\caravector_1,\dots,\caravector_h]\deq \{\sum_{j=1}^h\caraweightGj\caravectorj\ :\ \caraweightG\in[0,1]^h,\,\normloss{\caraweightG}=1\}$.
  % \begin{equation*}
  %   \operatorname{conv}_{\losspower}[\caravector_1,\dots,\caravector_h]\deq \Biggl\{\sum_{j=1}^h\caraweightGj\caravectorj\ :\ \caraweightG\in[0,1]^h,\,\normloss{\caraweightG}=1\Biggr\}\,.
  % \end{equation*}
Then, every vector $\boldsymbol{v}\in\operatorname{conv}_{\losspower}[\caravector_1,\dots,\caravector_h]$ can be written as $\boldsymbol{v}=  \sum_{j=1}^h\caraweightj\caravectorj$ 
% \begin{equation*}
%   \boldsymbol{v}=  \sum_{j=1}^h\caraweightj\caravectorj
% \end{equation*}
with $\caraweight\in[0,1]^h$, 
$\normloss{\caraweight} \leq 1$, 
and $\card{j\in\{1,\dots,h\}\,:\,\caraweightj \neq 0}\leq \tlnbrdata+1$.
\end{lemma}
\noindent 
The cardinality of a set $\mathcal A$ is denoted by $\card{\mathcal A}$ and the $\ell_{\losspower}$-norm of a vector $\boldsymbol{v}\in\R^a$ by $\normloss{\boldsymbol{v}}\deq(\sum_{j=1}^{a}\abs{v_j}^{\losspower})^{1/\losspower}$\label{def:lq}.
The lemma follows readily from the mentioned results; therefore, 
its proof is omitted.
The lemma states that  
every vector in the $\losspower$-convex hull is a $\losspower$-convex combination of at most $\tlnbrdata+1$~vectors from the set of vectors that generate the $\losspower$-convex hull. 
(Note that for $\losspower<1$,
our definition of the $\losspower$-convex hull is more restrictive than the standard definition---cf.~\citet[Remark on Page~142]{Bastero95}---but it leads to a concise statement and is sufficient for our purposes.)

\section{Discussion}\label{sec:discussion}
We have proved that the optimization landscapes of empirical-risk minimization over wide networks have no \problematics.
This finding is not surprising:
empirical evidence has long suggested that wide networks are comparatively easy to optimize.
Our main contribution is the fact that we underpin these empirical findings with mathematical theory.
More generally,
we contribute rigorous and detailed mathematical arguments and techniques to a discussion that is currently dominated by heuristics.

A key feature of the results themselves is their generality:
our framework allows for arbitrary depths,
for constraint as well as unconstraint estimation, 
essentially arbitrary activation,
and for a very wide spectrum of loss functions
%activation functions, 
and input and output data.
This generality demonstrates that the absence of \problematics\  is not a feature of specific estimators,  network functions, or data but, instead,
a universal property of wide networks.

Also the underlying techniques are novel:
the main idea of most approaches in the field is to construct basis functions for the networks; 
in contrast, we formulate parametrizations that make the networks easy to work with.
We could thus envision applications of our new concepts  beyond the presented topic.

A limitation of Theorem~\ref{mainresult}, our final result in the general case, is that it only applies to networks that are much wider than the networks currently used in practice.
The most critical aspect is the exponential dependence on the depth.
We have removed this dependence in Theorem~\ref{mainresultlinear} on linear activation,
but it remains unclear whether this improvement can be established more generally---see our discussion in Section~\ref{sec:mainresult}.
Despite this limitation,
our results give clear support for the hypothesis that increasing network widths makes optimization landscapes more benign.
This is the first time such clear support, especially in terms of generality,
has been established.

While we analyze the optimization landscape itself,
another line of research analyzes specific descent schemes.
For example, for certain settings,
\citet{Allen-Zhu19} and \citet{Du2019} show that (stochastic-)gradient descent can find global optima in linear time if the network widths are of some polynomial order in $\nbrsamples$ and~$\nbrlayers$ (they only consider $\nbroutput=1$).
Their rates  have additional terms that might depend on the depth (such as eigenvalues of intricate kernel matrices),
and their polynomials have unspecified factors and orders (compare to our clean formula $2\nbroutput(\nbrsamples+1)^{\nbrlayers}$),
which makes their bounds intractable in practice.
Moreover,
their setup does not account for regularization and constraints.
Nevertheless,
similarly to our theory,
their theories highlight the usefulness of wide networks in practice.

There is also current research about the role of constraints/regularization in deep learning---see 
\citep{Zhang2016}, for example.
Our Theorem~\ref{unconstraint} adds to this discussion by pointing out that explicit constraints/regularization are necessary for wide networks to be meaningful at all.
This finding supports, for example, the recent interest in sparsity in deep learning \citep{Frankle2019,Hebiri20,Hieber2017}.

Standard competitors to deep learning  are ``classical'' high-dimensional estimators, such as the ridge estimator and the lasso estimator~\citep{Hoerl70,Tibshirani96},
which have been studied in statistics extensively~\citep{Zhuang18}.
A common argument for these estimators is that their objective functions are convex or equipped with efficient algorithms for optimization~\citep{Bien18,Friedman2010},
but our results---and the research on this topic more generally---indicate that global optimization of the objective functions in deep learning can be feasible as well.

\ifarXiv
\subsection*{Acknowledgments}
We thank Yannick D\"uren, Shih-Ting Huang, Sophie-Charlotte Klose, Mike Laszkiewicz, Nils M\"uller, Mahsa Taheri, and Christoph Th\"ale for their valuable input.
\fi

%\newpage
\bibliography{Bibliography}

\newpage
\appendix
\section{Proofs}

\subsection{Proofs of Theorems~\ref{mainresult}, \ref{mainresultlinear}, and \ref{unconstraint} and an Extension}

\begin{proof}[Proof of Theorem~\ref{mainresult}]
The proof combines the main results of Sections~\ref{sec:pe} and \ref{sec:bl}. 
Let $\parameter\in\parameterS$ be an arbitrary parameter and $\parameterG\in\parameterS$ a global minimum of the objective function~\eqref{def:objective}.
In view of Proposition~\ref{res:characterization} in Section~\ref{sec:pe},
we need to show that $\parameter\pequi\parameterG$,
and this follows directly from Corollary~\ref{res:mainblock} in Section~\ref{sec:bl}.
\end{proof}

Figure~\ref{fig:minima} illustrates that a nonconvex objective function can have \problematics,
but Theorem~\ref{mainresult} proves that this is not the case  for deep learning with wide networks.
Figure~\ref{fig:minima} also illustrates that a nonconvex objective function can have ``disconnected'' global minima,
but, in view of Corollary~\ref{res:mainblock} and our results more generally,
we can rule out this case as well.
In other words,
the set of global minima is ``connected'' in the sense that it  is a set of path-constant parameters.
The connectedness of the global minima is of minor importance in practice but an interesting topological property nevertheless.

\begin{proof}[Proof of Theorem~\ref{mainresultlinear}]
The proof is exactly the same as the proof of Theorem~\ref{mainresult} with one exception:
In the case of linear activation, 
the matrices in the proof of Proposition~\ref{res:reparametrization} can be merged;
for example,
\begin{equation*}
\parameterl\activation^{\nbrlayers}\Bigl[\parameterE^{\nbrlayers-1}\activation^{\nbrlayers-1}\bigl[\parameterE^{\nbrlayers-2}\cdots \activation^1[\parameterE^0\data]\bigr]\Bigr]=(\parameterl\parameterE^{\nbrlayers-1})\activation^{\nbrlayers-1}\bigl[\parameterE^{\nbrlayers-2}\cdots \activation^1[\parameterE^0\data]\bigr]\,.
\end{equation*}
Hence, rather than the output dimension of $\parameterj$,
the relevant dimension in the induction step is the output dimension of~$\parameterl$;
in other words, 
the induction step does not accumulate the dimensions of the matrices.
This observation improves the term $\nbroutput(\nbrsamples+1)^{\nbrlayers}$ in Proposition~\ref{res:reparametrization} to $\nbroutput(\nbrsamples+1)$,
which can then be translated into an improved lower bound in the theorem.
\end{proof}

\begin{proof}[Proof of Theorem~\ref{unconstraint}]
The proof can be readily established by repurposing the arguments from \citet{Li2018}.
\end{proof}

\subsection{Proof of Proposition~\ref{res:characterization}}

\begin{proof}[Proof of Proposition~\ref{res:characterization}]
The key idea is to exploit the properties of monotone functions and  convex  functions.
Let $\parameter\in\parameterS$ be an arbitrary parameter 
and $\parameterG\in\parameterS$ a global minimum of the objective function~\eqref{def:objective} such that $\parameter\pequi\parameterG$.
Let $\parameter',\parameterG'\in\parameterS$ and $\pathfunction_{\parameter',\parameterG'}\,:\,[0,1]\to\parameterS$ be as described in Definition~\ref{def:path}.
Note first that, since $\parameterG\pdec\parameterG'$ (see Definition~\ref{def:path}),
the parameter~$\parameterG'$ is also a global minimum of the objective function.
We then use
1.~simple algebra,
2.~the assumed convexity of the function~$\paradd\mapsto\lossF{\network_{\pathfunction_{\parameter',\parameterG'}}}$ (see Definition~\ref{def:path} again),
3.~the assumed endpoints of the function~$\pathfunction_{\parameter',\parameterG'}$ (see Definition~\ref{def:path} once more),
 4.~the fact that $\parameterG'$ is a global minimum of~\eqref{def:objective}, 
and 5.~the fact that $(1-\paradd)+\paradd=1$
to derive for all $\paradd\in[0,1]$ that 
\begin{align*}
  \lossFB{\network_{\pathfunction_{\parameter',\parameterG'}[\paradd]}}&=  \lossFB{\network_{\pathfunction_{\parameter',\parameterG'}[(1-\paradd)\cdot0+\paradd\cdot 1]}}\\
  &\leq (1-\paradd)\lossFB{\network_{\pathfunction_{\parameter',\parameterG'}[0]}} + \paradd\lossFB{\network_{\pathfunction_{\parameter',\parameterG'}[1]}}\\
  &= (1-\paradd)\lossFB{\network_{\pathfunction_{\parameter',\parameterG'}[0]}} + \paradd\lossF{\network_{\parameterG'}}\\
  &\leq (1-\paradd)\lossFB{\network_{\pathfunction_{\parameter',\parameterG'}[0]}} + \paradd\lossFB{\network_{\pathfunction_{\parameter',\parameterG'}[0]}}\\
  &= \lossFB{\network_{\pathfunction_{\parameter',\parameterG'}[0]}}\,.
\end{align*}
Assume first that the inequality is strict: 
$\lossF{\network_{\pathfunction_{\parameter',\parameterG'}[\paraddP]}}<\lossF{\network_{\pathfunction_{\parameter',\parameterG'}[0]}}$ for a $\paraddP\in[0,1]$.
Then, by the assumed convexity of $\paradd\mapsto \lossF{\network_{\pathfunction_{\parameter',\parameterG'}[\paradd]}}$ (see Definition~\ref{def:path}) and Lemma~\ref{res:convex},
there is a $\paraddPP\in\argmin_{\paradd\in(0,1]}\{\lossF{\network_{\pathfunction_{\parameter',\parameterG'}[\paradd]}}\}$ 
such that $\pathfunctionP\,:\,[0,1]\to\R$ defined through $\pathfunctionPF{\paradd}\deq \lossF{\network_{\pathfunction_{\parameter',\parameterG'}[\paraddPP\paradd]}}$ for all $\paradd\in[0,1]$ is nonincreasing, $\pathfunctionPF{1}$ is a global minimum of the objective function~\eqref{def:objective},
and $\pathfunctionPF{0}=\lossF{\network_{\pathfunction_{\parameter',\parameterG'}[0]}}=\lossF{\network_{\parameter'}}>\pathfunctionPF{1}=\lossF{\network_{\pathfunction_{\parameter',\parameterG'}[\paraddPP]}}$.
Hence, the function $\overline\pathfunction\,:\,[0,1]\to\parameterS$ defined through $\overline\pathfunction[\paradd]\deq \pathfunction_{\parameter',\parameterG'}[\paraddPP\paradd]$ for all $\paradd\in[0,1]$ is a function that satisfies the conditions of Definition~\ref{def:spurious} for the parameter~$\parameter'$ and the global minimum~$\pathfunctionPF{1}$. 
Combining this result with the assumed relationship $\parameter\pdec\parameter'$ (see Definition~\ref{def:path}) yields---see Definition~\ref{def:spurious}---the fact that~\parameter\ is not a \problematic\  of the objective function~\eqref{def:objective}.

We can thus assume that  $\paradd\mapsto\lossF{\network_{\pathfunction_{\parameter',\parameterG'}[\paradd]}}$ is constant,
which implies $\parameter'\pdec\parameterG'$ by Definition~\ref{def:path}.
The fact that $\parameter\pdec\parameter'$ (see Definition~\ref{def:path} again) and the transitivity of the path constantness (see Property~3 in Lemma~\ref{res:basic}) then yield the fact that~$\parameter\pdec\parameterG'$.
Hence,
$\parameter$ is a global minimum and, therefore, not a \problematic---see Definition~\ref{def:spurious} again.
\end{proof}

\subsection{Proof of Proposition~\ref{res:reparametrization}}

\begin{proof}[Proof of Proposition~\ref{res:reparametrization}]
Our proof strategy is to apply Lemma~\ref{twolayer}, which is designed for one individual layer,
layer by layer.
We first introduce some convenient notation.
We define, with some abuse of notation,
$\activation^j\,:\,\R^{\nbrparameterj\times\nbrsamples}\to\R^{\nbrparameterj\times\nbrsamples}$
through $(\activation^j[\matrixG])_{uv}\deq\activationindjF{\matrixG_{uv}}$ for all $j\in\{1,\dots,\nbrlayers\}$, $u\in\{1,\dots,\nbrparameterj\}$,  $v\in\{1,\dots,\nbrsamples\}$,  and $\matrixG \in\R^{\nbrparameterj\times\nbrsamples}$.
We also define the data matrix $\data\in\R^{\nbrinput\times\nbrsamples}$ through $\data_{ji}\deq (\vectorini)_j$ for all $j\in\{1,\dots,\nbrinput\}$ and $i\in\{1,\dots,\nbrsamples\}$, that is, each column of~\data\ consists of one sample.
We finally write 
 \begin{equation*}
 \network_{\parameter}[\data]\deq \bigl(\networkAF{\vectorin_1},\dots,\networkAF{\vectorin_{\nbrsamples}}\bigr)= \parameterl\activation^{\nbrlayers}\bigl[\parameterE^{\nbrlayers-1}\cdots \activation^1[\parameterE^0\data]\bigr]\in\R^{\nbroutput\times\nbrsamples}~~~~~~\text{for all}~\parameter\in\parameterS\,.
\end{equation*}
Hence, $\network_{\parameter}[\data]$ summarizes the network's outputs for the given data.

Given a parameter~$\parameter\in\parameterS$,
we establish a corresponding upper-block parameter~$\parameterU\in\shapeupper$ layer by layer, starting from the outermost layer.
We write
\begin{equation*}
\networkAF{\data}= \underbrace{\parameterl}_{=:\shmatrixout\in\R^{\nbrparameterll\times\nbrparameterl}}\underbrace{\activation^{\nbrlayers}}_{=:\shfunction}\Bigl[\underbrace{\parameterE^{\nbrlayers-1}}_{=:\shmatrixin\in\R^{\nbrparameterl\times\nbrparameter^{\nbrlayers-1}}}\underbrace{\activation^{\nbrlayers-1}\bigl[\parameterE^{\nbrlayers-2}\cdots \activation^1[\parameterE^0\data]\bigr]}_{=:\shmatrixdata\in\R^{\nbrparameter^{\nbrlayers-1}\times\nbrsamples}}\Bigr]\,.
\end{equation*}
Lemma~\ref{twolayer} for two-layer networks then gives (by Lemma~\ref{res:symmetry}, we can assume without loss of generality the fact that~\perm\ is the identity function, that is, $\shmatrixout_{\perm}=\shmatrixout$ and $\shmatrixin^{\perm}=\shmatrixin$)
\begin{equation*}
\networkAF{\data}=   \parameterUl\activation^{\nbrlayers}\Biggl[{\parameterUUE^{\nbrlayers-1} \choose \zero}\activation^{\nbrlayers-1}\bigl[\parameterE^{\nbrlayers-2}\cdots \activation^1[\parameterE^0\data]\bigr]\Biggr]
\end{equation*}
for a matrix $\parameterUl\in\R^{\nbrparameterll\times\nbrparameterl}$ that satisfies $\regularizerRF{\parameterUl,1}\leq \regularizerRF{\parameterl,1}$ (recall the definition of~\regularizerR\ on Page~\pageref{def:regularizerbasic}) 
and meets Condition~3 in the first part of Definition~\ref{def:blocks} on block parameters as long as $\shapefactor\geq \nbrparameterll(\nbrsamples+1)=\nbroutput(\nbrsamples+1)$,
and for a matrix $\parameterUUE^{\nbrlayers-1}\in\R^{\nbroutput(\nbrsamples+1)\times \nbrparameter^{\nbrlayers-1}}$ that satisfies $\regularizerRF{\parameterUUE^{\nbrlayers-1},1}\leq \regularizerRF{\parameterE^{\nbrlayers-1},1}$ (or $\regularizerRF{\parameterUUE^{\nbrlayers-1},q}\leq \regularizerRF{\parameterE^{\nbrlayers-1},q}$ if $\nbrlayers=1$) and consists of the first $\nbroutput(\nbrsamples+1)$ rows  of the matrix~$\parameterE^{\nbrlayers-1}$.
(We implicitly assume here and in the following  $\nbrparameter^j\geq \nbroutput(\nbrsamples+1)^{\nbrlayers-j+1}$ for all $j\in\{1,\dots,\nbrlayers\}$---which is  the generic case in view of Corollary~\ref{res:mainblock}---to keep the notation manageable,
but extending the proof to the general  case is straightforward.) 

Now, define a parameter $\parameterG^{\nbrlayers}\in\parameterSFull$ 
through 
\begin{equation*}
  \parameterG^{\nbrlayers}\deq (\parameterUl,\parameterE^{\nbrlayers-1},\ldots,\parameterE^0)
\end{equation*}
and a function $\pathfunction_{\parameter,\parameterG^{\nbrlayers}}\,:\,[0,1]\to\parameterSFull$
through
\begin{equation*}
  \pathfunction_{\parameter,\parameterG^{\nbrlayers}}[\paradd]\deq(1-\paradd)\parameter+\paradd\parameterG^{\nbrlayers}~~~~~~~~~~~\text{for all}~\paradd\in[0,1]\,.
\end{equation*}
The function $\pathfunction_{\parameter,\parameterG^{\nbrlayers}}$ is continuous and satisfies $\pathfunction_{\parameter,\parameterG^{\nbrlayers}}[0]=\parameter$ and 
$\pathfunction_{\parameter,\parameterG^{\nbrlayers}}[1]=\parameterG^{\nbrlayers}$. 
Moreover,
we can 
1.~use the definitions of the function~$\pathfunction_{\parameter,\parameterG^{\nbrlayers}}$ and  the networks,
2.~split the network along the outermost layer,
3.~invoke the block shape of~\parameterUl\ and the definition of~$\parameterUUE^{\nbrlayers-1}$ as the $\nbroutput(\nbrsamples+1)$ first  rows  of the matrix~$\parameterE^{\nbrlayers-1}$,
4.~use the above-stated inequalities for the network~$\networkAF{\data}$,
and 5.~consolidate the terms
to show for all $\paradd\in[0,1]$ that
\begin{align*}
\network_{\pathfunction_{\parameter,\parameterG^{\nbrlayers}}[\paradd]}[\data]
  &=  \bigl((1-\paradd)\parameterl+\paradd\parameterUl\bigr)\activation^{\nbrlayers}\bigl[\parameterE^{\nbrlayers-1}\cdots \activation^1[\parameterE^0\data]\bigr]\\
  &=  (1-\paradd)\parameterl\activation^{\nbrlayers}\bigl[\parameterE^{\nbrlayers-1}\cdots \activation^1[\parameterE^0\data]\bigr]+\paradd\parameterUl\activation^{\nbrlayers}\bigl[\parameterE^{\nbrlayers-1}\cdots \activation^1[\parameterE^0\data]\bigr]\\
  &=  (1-\paradd)\parameterl\activation^{\nbrlayers}\bigl[\parameterE^{\nbrlayers-1}\cdots \activation^1[\parameterE^0\data]\bigr]+\paradd\parameterUl\activation^{\nbrlayers}\biggl[{\parameterUUE^{\nbrlayers-1} \choose \zero}\cdots \activation^1[\parameterE^0\data]\biggr]\\
  &=  (1-\paradd)\networkAF{\data}+\paradd\networkAF{\data}\\
&=\networkAF{\data}\,.
\end{align*}
Hence, the function $\paradd\mapsto \lossF{\network_{\pathfunction_{\parameter,\parameterG^{\nbrlayers}}[\paradd]}}$ is constant.
Finally,
we use
1.~the definition of the constraint in~\eqref{def:constraint},
2.~the definition of the function~$\pathfunction_{\parameter,\parameterG^{\nbrlayers}}$,
3.~the convexity of the $\ell_1$-norm,
4.~the above stated fact that $\regularizerRF{\parameterUl,1}\leq \regularizerRF{\parameterl,1}$,
5.~a consolidation,
6.~again the definition of the regularizer,
and 7.~the fact that $\parameter\in\parameterS$ to show for all $\paradd\in[0,1]$ that
\begingroup
\allowdisplaybreaks
\begin{align*}
  \regularizerFB{\pathfunction_{\parameter,\parameterG^{\nbrlayers}}[\paradd]}
  &=\max\Bigl\{\tuningparameterout\max_{j\in\{1,\dots,\nbrlayers\}}\normoneMB{\bigl(\pathfunction_{\parameter,\parameterG^{\nbrlayers}}[\paradd]\bigr)^j},\tuningparameterin\normqMB{\bigl(\pathfunction_{\parameter,\parameterG^{\nbrlayers}}[\paradd]\bigr)^0}\Bigr\}\\
&=\max\Bigl\{\tuningparameterout\max_{j\in\{1,\dots,\nbrlayers-1\}}\normoneM{\parameterj},\tuningparameterout\normoneM{(1-\paradd)\parameterl+\paradd\parameterUl},\tuningparameterin\normqM{\parameterz}\Bigr\}\\
&\leq\max\Bigl\{\tuningparameterout\max_{j\in\{1,\dots,\nbrlayers-1\}}\normoneM{\parameterj},(1-\paradd)\tuningparameterout\normoneM{\parameterl}+\paradd\tuningparameterout\normoneM{\parameterUl},\tuningparameterin\normqM{\parameterz}\Bigr\}\\
&\leq\max\Bigl\{\tuningparameterout\max_{j\in\{1,\dots,\nbrlayers-1\}}\normoneM{\parameterj},(1-\paradd)\tuningparameterout\normoneM{\parameterl}+\paradd\tuningparameterout\normoneM{\parameterl},\tuningparameterin\normqM{\parameterz}\Bigr\}\\
&=\max\Bigl\{\tuningparameterout\max_{j\in\{1,\dots,\nbrlayers-1\}}\normoneM{\parameterj},\tuningparameterout\normoneM{\parameterl},\tuningparameterin\normqM{\parameterz}\Bigr\}\\
&=\regularizerF{\parameter}\\
&\leq 1\,.
\end{align*}
\endgroup
Hence, $\pathfunction_{\parameter,\parameterG^{\nbrlayers}}[\paradd]\in\parameterS$ for all $\paradd\in[0,1]$.
In conclusion, we have shown---see Definition~\ref{def:path}---that~$\parameter$ and~$\parameterG^{\nbrlayers}$ are path constant: $\parameter\pdec\parameterG^{\nbrlayers}$.

We then move one layer inward.
Lemma~\ref{twolayer} ensures that (recall again Lemma~\ref{res:symmetry})
\begin{equation*}
  \underbrace{\parameterUUE^{\nbrlayers-1}}_{\shmatrixout\in\R^{\nbroutput(\nbrsamples+1)\times \nbrparameter^{\nbrlayers-1}}}\underbrace{\activation^{\nbrlayers-1}}_{\shfunction}\bigl[\underbrace{\parameterE^{\nbrlayers-2}}_{\shmatrixin\in\R^{\nbrparameter^{\nbrlayers-1}\times\nbrparameter^{\nbrlayers-2}}}\underbrace{\cdots \activation^1[\parameterE^0\data]}_{\shmatrixdata\in\R^{\nbrparameter^{\nbrlayers-2}\times \nbrsamples}}\bigr]=\parameterUUUE^{\nbrlayers-1}\activation^{\nbrlayers-1}\Biggl[{\parameterUUE^{\nbrlayers-2} \choose \zero}\cdots \activation^1[\parameterE^0\data]\Biggr]
\end{equation*}
for a matrix $\parameterUUUE^{\nbrlayers-1}\in\R^{\nbroutput(\nbrsamples+1)\times\nbrparameter^{\nbrlayers-1}}$ that satisfies $\regularizerRF{\parameterUUUE^{\nbrlayers-1},1}\leq \regularizerRF{\parameterUUE^{\nbrlayers-1},1}\leq \regularizerRF{\parameterE^{\nbrlayers-1},1}$
and meets Condition~3 in the first part of Definition~\ref{def:blocks} on block parameters as long as $\shapefactor\geq \nbroutput(\nbrsamples+1)^2$,
and for a matrix $\parameterUUE^{\nbrlayers-2}\in\R^{\nbroutput(\nbrsamples+1)^2\times \nbrparameter^{\nbrlayers-2}}$ that satisfies $\regularizerRF{\parameterUUE^{\nbrlayers-2},1}\leq \regularizerRF{\parameterE^{\nbrlayers-2},1}$ (or $\regularizerRF{\parameterUUE^{\nbrlayers-2},q}\leq \regularizerRF{\parameterE^{\nbrlayers-2},q}$ if $\nbrlayers=2$) and consists of the first $\nbroutput(\nbrsamples+1)^2$ rows  of the matrix~$\parameterE^{\nbrlayers-2}$. 

Next, we define $\parameterUE^{\nbrlayers-1}\in\R^{\nbrparameterl\times\nbrparameter^{\nbrlayers-1}}$ through $(\parameterUE^{\nbrlayers-1})_{uv}\deq(\parameterUUUE^{\nbrlayers-1})_{uv}$ for $u\leq \nbroutput(\nbrsamples+1)$ and $(\parameterUE^{\nbrlayers-1})_{uv}\deq0$ otherwise.
Combining this definition with the above-derived results yields
\begin{equation*}
\networkAF{\data}=   \parameterUl\activation^{\nbrlayers}\Biggl[\parameterUE^{\nbrlayers-1} \activation^{\nbrlayers-1}\Biggl[{\parameterUUE^{\nbrlayers-2} \choose \zero}\cdots \activation^1[\parameterE^0\data]\Biggr]\Biggr]\,,
\end{equation*}
and the matrix~$\parameterUE^{\nbrlayers-1}$
satisfies 
$\regularizerRF{\parameterUE^{\nbrlayers-1},1}=\regularizerRF{\parameterUUUE^{\nbrlayers-1},1}\leq \regularizerRF{\parameterE^{\nbrlayers-1},1}$ and meets Condition~2 in the first part of Definition~\ref{def:blocks} on block parameters as long as $\shapefactor\geq \nbroutput(\nbrsamples+1)^2$.

Similarly as above, define a parameter $\parameterG^{\nbrlayers-1}\in\parameterSFull$ 
through 
\begin{equation*}
  \parameterG^{\nbrlayers-1}\deq (\parameterUl,\parameterUE^{\nbrlayers-1},\parameterE^{\nbrlayers-2},\ldots,\parameterE^0)
\end{equation*}
and a function $\pathfunction_{\parameterG^{\nbrlayers},\parameterG^{\nbrlayers-1}}\,:\,[0,1]\to\parameterSFull$
through
\begin{equation*}
  \pathfunction_{\parameterG^{\nbrlayers},\parameterG^{\nbrlayers-1}}[\paradd]\deq(1-\paradd)\parameterG^{\nbrlayers}+\paradd\parameterG^{\nbrlayers-1}~~~~~~~~~~~\text{for all}~\paradd\in[0,1]
\end{equation*}
to show that $\parameterG^{\nbrlayers}\pdec\parameterG^{\nbrlayers-1}$.
In view of Property~3 in Lemma~\ref{res:basic},
we can conclude that $\parameter\pdec\parameterG^{\nbrlayers-1}$.

Finish the proof by induction over the layers,
and note that the lower-block parameters can be established in the same way.
\end{proof}

\subsection{Proof of Proposition~\ref{res:path}}

\begin{proof}[Proof of Proposition~\ref{res:path}]
The key ingredient of the proof is the assumed block structure of the parameters.
Consider two block parameters $\parameter\in\shapeupper$ and 
 $\parameterG\in\shapelower$ and 
define the parameters
\begin{align*}
    \parameter'&\deq(\parameterl,\parameterE^{\nbrlayers-1}+\parameterGE^{\nbrlayers-1},\dots,\parameterE^0+\parameterGE^0)\in\parameterSFull\\
    \parameterG'&\deq(\parameterGl,\parameterE^{\nbrlayers-1}+\parameterGE^{\nbrlayers-1},\dots,\parameterE^0+\parameterGE^0)\in\parameterSFull
\end{align*}
and the function
\begin{align*}
  \pathfunction_{\parameter',\parameterG'}\ :\ [0,1]\,&\to\,\parameterSFull\\
\paradd\,&\mapsto\,(1-\paradd)\parameter'+\paradd\parameterG'=\bigl((1-\paradd)\parameterl+\paradd\parameterGE^{\nbrlayers},\parameterE^{\nbrlayers-1}+\parameterGE^{\nbrlayers-1},\dots,\parameterz+\parameterGE^{0}\bigr)\,.
\end{align*}
By the row-wise structure of the constraint (see Page~\pageref{def:constraint} again),
the convexity of the $\ell_1$-norm,
 and the block shapes of the parameters (see Figure~\ref{fig:shapes} again),
we can find that $\regularizerF{ \pathfunction_{\parameter',\parameterG'}[\paradd]}\leq \max\{\regularizerF{\parameter},\regularizerF{\parameterG}\}\leq 1$ for all~$\paradd\in[0,1]$,
that is, $ \pathfunction_{\parameter',\parameterG'}[\paradd]\in\parameterS$ for all $\paradd\in[0,1]$.
One can also verify readily the fact that the function~$\pathfunction_{\parameter',\parameterG'}$ is continuous, $\pathfunction_{\parameter',\parameterG'}[0]=\parameter'$, and 
$\pathfunction_{\parameter',\parameterG'}[1]=\parameterG'$.
Next, we define
\begin{equation*}
  \transitP\deq (\transita\parameterl+\transitb\parameterGl,\parameterE^{\nbrlayers-1}+\parameterGE^{\nbrlayers-1},\dots,\parameterE^0+\parameterGE^0)\in\parameterSFull\text{~~~~~~for all~}\transita,\transitb\in\R\,,
\end{equation*}
which generalizes~$\pathfunction_{\parameter',\parameterG'}$ in the sense that $\pathfunction_{\parameter',\parameterG'}[\paradd]=[\parameter',\parameterG']_{1-\paradd,\paradd}$ for all $\paradd\in[0,1]$.
 We then 
1.~invoke the definition of \transitP\ and the definition of the networks in~\eqref{networks},
2.~split the network along the outer layer,
3.~use the block structures of the parameters and the assumption that 
$\activation^j[\boldsymbol{b}]=(\activationindjF{b_1},\dots,\activationindjF{b_{\nbrparameterj}})\tp$ for all $j\in\{1,\dots,\nbrlayers\}$ and $\boldsymbol{b}\in\R^{\nbrparameterj}$,
4.~continue in this fashion,
and 5.~invoke the definition of the networks in~\eqref{networks} again to find for all $\transita,\transitb\in\R$ and $\vectorin\in\domainin$ that 
\begingroup
\allowdisplaybreaks
\begin{align*}
 &  \network_{\transitP}[\vectorin]\\
 &=(\transita\parameterl+\transitb\parameterGl)\activation^{\nbrlayers}\biggl[(\parameterE^{\nbrlayers-1}+\parameterGE^{\nbrlayers-1})\activation^{\nbrlayers-1}\Bigl[\cdots\activation^1\bigl[(\parameterE^0+\parameterGE^0)\vectorin\bigr]\Bigr]\biggr]\\
 &=\transita\parameterl\activation^{\nbrlayers}\biggl[(\parameterE^{\nbrlayers-1}+\parameterGE^{\nbrlayers-1})\activation^{\nbrlayers-1}\Bigl[\cdots\activation^1\bigl[(\parameterE^0+\parameterGE^0)\vectorin\bigr]\Bigr]\biggr]\\
 &\hspace{40mm}~~~+\transitb\parameterGE^{\nbrlayers}\activation^{\nbrlayers}\biggl[(\parameterE^{\nbrlayers-1}+\parameterGE^{\nbrlayers-1})\activation^{\nbrlayers-1}\Bigl[\cdots\activation^1\bigl[(\parameterE^0+\parameterGE^0)\vectorin\bigr]\Bigr]\biggr]\\
 &=\transita\parameterl\activation^{\nbrlayers}\biggl[\parameterE^{\nbrlayers-1}\activation^{\nbrlayers-1}\Bigl[\cdots\activation^1\bigl[(\parameterE^0+\parameterGE^0)\vectorin\bigr]\Bigr]\biggr]+\transitb\parameterGE^{\nbrlayers}\activation^{\nbrlayers}\biggl[\parameterGE^{\nbrlayers-1}\activation^{\nbrlayers-1}\Bigl[\cdots\activation^1\bigl[(\parameterE^0+\parameterGE^0)\vectorin\bigr]\Bigr]\biggr]\\
&=\cdots=\transita\parameterE^{\nbrlayers}\activation^{\nbrlayers}\biggl[\parameterE^{\nbrlayers-1}\activation^{\nbrlayers-1}\Bigl[\cdots \activation^1\bigl[\parameterE^0\vectorin\bigr]\Bigr]\biggr]+\transitb\parameterGE^{\nbrlayers}\activation^{\nbrlayers}\biggl[\parameterGE^{\nbrlayers-1}\activation^{\nbrlayers-1}\Bigl[\cdots\activation^1\bigl[\parameterGE^0\vectorin\bigr]\Bigr]\biggr]\\
&=\transita\networkAF{\vectorin}+\transitb\networkAGF{\vectorin}\,.
\end{align*} 
\endgroup
Finally,
we  use 
1.~the definition of the function~$\pathfunction_{\parameter',\parameterG'}$ and of the parameter $\transitP$,
2.~the above display  with $\transita=1-(1-\paraddw)\paraddo-\paraddw\paraddt$ and $\transitb=(1-\paraddw)\paraddo+\paraddw\paraddt$,
3.~a rearrangement of the terms,
4.~the assumed convexity of the loss function~$\loss$,
5.~again the above display,
and 6.~again the definition of $\pathfunction_{\parameter',\parameterG'}$,
to find for all $\paraddw,\paraddo,\paraddt\in[0,1]$ that
\begin{align*}
  &\lossFB{\network_{\pathfunction_{\parameter',\parameterG'}[(1-\paraddw)\paraddo+\paraddw\paraddt]}}\\
  &=\lossFB{\network_{\transitwP}}\\
  &=\lossFBB{\bigl(1-(1-\paraddw)\paraddo-\paraddw\paraddt\bigr)\network_{\parameter}+\bigl((1-\paraddw)\paraddo+\paraddw\paraddt\bigr)\network_{\parameterG}}\\
  &=\lossFBB{\bigl((1-\paraddw)-(1-\paraddw)\paraddo\bigr)\network_{\parameter}+(\paraddw-\paraddw\paraddt)\network_{\parameter}+(1-\paraddw)\paraddo\network_{\parameterG}+\paraddw\paraddt\network_{\parameterG}}\\
  &=\lossFBB{(1-\paraddw)\bigl((1-\paraddo)\network_{\parameter}+\paraddo\network_{\parameterG}\bigr)+\paraddw\bigl((1-\paraddt)\network_{\parameter}+\paraddt\network_{\parameterG}\bigr)}\\
  &\leq (1-\paraddw)\lossFB{(1-\paraddo)\network_{\parameter}+\paraddo\network_{\parameterG}}+\paraddw\lossFB{(1-\paraddt)\network_{\parameter}+\paraddt\network_{\parameterG}}\\
 &= (1-\paraddw)\lossFB{\network_{\transitpoP}}+\paraddw\lossFB{\network_{\transitptP}}\\
 &=  (1-\paraddw)\lossFB{\network_{\pathfunction_{\parameter',\parameterG'}[\paraddo]}}+ \paraddw\lossFB{\network_{\pathfunction_{\parameter',\parameterG'}[\paraddt]}}\,,
\end{align*}
which means that 
$\paradd\mapsto \lossF{\network_{\pathfunction_{\parameter',\parameterG'}[\paradd]}}$ is convex.
We conclude---see Definition~\ref{def:path}---that~$\parameter'\pcon\parameterG'$.

In view of Definition~\ref{def:path},
it is left to show that $\parameter\pdec\parameter'$ and $\parameterG\pdec\parameterG'$.
Consider now the function
\begin{align*}
  \pathfunction_{\parameter,\parameter'}\ :\ [0,1]\,&\to\,\parameterSFull\\
\paradd\,&\mapsto\,(\parameterl,\parameterE^{\nbrlayers-1}+\paradd\parameterGE^{\nbrlayers-1},\dots,\parameterE^{0}+\paradd\parameterGE^{0})\,.
\end{align*}
By the row-wise structure of the constraint (see Page~\pageref{def:constraint} once more) and the block shapes of the parameters (see Figure~\ref{fig:shapes} once more),
we can find that $\regularizerF{ \pathfunction_{\parameter,\parameter'}[\paradd]}\leq \max\{\regularizerF{\parameter},\regularizerF{\paradd\parameterG}\}\leq 1$ for all~$\paradd\in[0,1]$,
that is, $ \pathfunction_{\parameter,\parameter'}[\paradd]\in\parameterS$ for all $\paradd\in[0,1]$.
One can also verify readily the fact that the function~$\pathfunction_{\parameter,\parameter'}$ is continuous, $\pathfunction_{\parameter,\parameter'}[0]=\parameter$, and 
$\pathfunction_{\parameter,\parameter'}[1]=\parameter'$.
Moreover, we can
1.~invoke the definition of $\pathfunction_{\parameter,\parameter'}$ and the definition of the networks in~\eqref{networks},
2.~use the block structures of the parameters and the elementwise structure of the activation function,
3.~continue in this fashion,
and 4.~invoke the definitions of the function $\pathfunction_{\parameter,\parameter'}$ and the networks in~\eqref{networks} again to find for all $\paradd\in[0,1]$ and $\vectorin\in\domainin$ that  
\begingroup
\allowdisplaybreaks
\begin{align*}
  \network_{\pathfunction_{\parameter,\parameter'}[\paradd]}[\vectorin]
 &=\parameterl\activation^{\nbrlayers}\biggl[(\parameterE^{\nbrlayers-1}+\paradd\parameterGE^{\nbrlayers-1})\activation^{\nbrlayers-1}\Bigl[\cdots\activation^1\bigl[(\parameterE^0+\paradd\parameterGE^0)\vectorin\bigr]\Bigr]\biggr]\\
 &=\parameterl\activation^{\nbrlayers}\biggl[\parameterE^{\nbrlayers-1}\activation^{\nbrlayers-1}\Bigl[\cdots\activation^1\bigl[(\parameterE^0+\paradd\parameterGE^0)\vectorin\bigr]\Bigr]\biggr]\\
&=\cdots=\parameterE^{\nbrlayers}\activation^{\nbrlayers}\biggl[\parameterE^{\nbrlayers-1}\activation^{\nbrlayers-1}\Bigl[\cdots \activation^1\bigl[\parameterE^0\vectorin\bigr]\Bigr]\biggr]\\
&=  \network_{\pathfunction_{\parameter,\parameter'}[0]}[\vectorin]\,,
\end{align*} 
\endgroup
which implies that the function $\paradd\mapsto \lossF{\network_{\pathfunction_{\parameter,\parameter'}[\paradd]}}$ is constant.
We conclude that---see Definition~\ref{def:path}---that~$\parameter\pdec\parameter'$.

We can show in a similar way that $\parameterG\pdec\parameterG'$.
Hence, given Definition~\ref{def:path}, we find that $\parameter\pequi\parameterG$, as desired.
\end{proof} 

\subsection{Proof of Lemma~\ref{twolayer}}

\begin{proof}[Proof of Lemma~\ref{twolayer}]
The proof is essentially a careful reparametrization.
We proceed in three steps: see Figure~\ref{fig:twolayer} for an overview.

\emph{Step~1:} 
Fix a $k\in\{1,\dots,\tlnbrout\}$.
We first show that there is a matrix $\shmatrixoutS\in\R^{\tlnbrout\times \tlnbrmed}$ such that
\begin{enumerate}
\item $\shmatrixoutS\shfunctionF{\shmatrixin\shmatrixdata}=\shmatrixout\shfunctionF{\shmatrixin\shmatrixdata}$;
\item $\regularizerRF{\shmatrixoutS,\losspowerout}\leq \regularizerRF{\shmatrixout,\losspowerout}$;
\item $\card{j\in\{1,\dots,\tlnbrmed\}\,:\,\shmatrixoutS_{kj}\neq 0}\leq \tlnbrdata+1$;
\item $\shmatrixoutS_{a j}=\shmatrixout_{a j}$ for all $a\neq k$.
\end{enumerate}
Hence, we replace the matrix~\shmatrixout, which contains the ``output parameters,'' 
by a matrix whose $k$th row has at most~$\tlnbrdata+1$ nonzero entries---see the illustration in Figure~\ref{fig:twolayer}.

The proof of this step is based on our version of Carath\'eodory's theorem in Lemma~\ref{res:caraold}.
For every  $k\in\{1,\dots,\tlnbrout\}$ and $i\in\{1,\dots,\tlnbrdata\}$,
elementary matrix algebra yields that
\begin{equation*}
  \bigl(\shmatrixout\shfunctionF{\shmatrixin\shmatrixdata}\bigr)_{ki}=\sum_{j=1}^{\tlnbrmed}\shmatrixout_{kj}\bigl(\shfunctionF{\shmatrixin\shmatrixdata}\bigr)_{ji}\,.
\end{equation*}
Denoting the row vectors of a matrix~$\matrixG\in\R^{a\times b}$ by $\matrixG_{1\bullet},\dots,\matrixG_{a\bullet}\in\R^b$,
we then get 
\begin{equation*}
  \bigl(\shmatrixout\shfunctionF{\shmatrixin\shmatrixdata}\bigr)_{k\bullet}=\sum_{j=1}^{\tlnbrmed}\shmatrixout_{kj}\bigl(\shfunctionF{\shmatrixin\shmatrixdata}\bigr)_{j\bullet}\,.
\end{equation*}
The case $\normlossout{\shmatrixout_{k\bullet}}=0$ is straightforward to deal with: 
all elements of $\shmatrixout_{k\bullet}$ are then equal to zero,
and we can just set $\shmatrixoutS\deq \shmatrixout$.

We can thus assume the fact that $\normlossout{\shmatrixout_{k\bullet}}\neq 0$.
We then get from the preceding equality that
\begin{equation*}
    \bigl(\shmatrixout\shfunctionF{\shmatrixin\shmatrixdata}\bigr)_{k\bullet}=\sum_{j=1}^{\tlnbrmed}\underbrace{\frac{\shmatrixout_{kj}}{\normlossout{\shmatrixout_{k\bullet}}}}_{=: \caraweightGj}\underbrace{\normlossout{\shmatrixout_{k\bullet }}\bigl(\shfunctionF{\shmatrixin \shmatrixdata}\bigr)_{j\bullet}}_{=:\caravectorj}\,.
\end{equation*}
Since 
\begin{multline*}
\normlossout{\caraweightG}=  \Biggl(\sum_{j=1}^{\tlnbrmed}\abs{\caraweightGj}^{\losspowerout}\Biggr)^{1/\losspowerout}=\Biggl(\sum_{j=1}^{\tlnbrmed}\absBBB{\frac{\shmatrixout_{kj}}{\normlossout{\shmatrixout_{k\bullet}}}}^{\losspowerout}\Biggr)^{1/\losspowerout}\\
=\frac{1}{\normlossout{\shmatrixout_{k\bullet}}}\Biggl(\sum_{j=1}^{\tlnbrmed}\abs{\shmatrixout_{kj}}^{\losspowerout}\Biggr)^{1/\losspowerout}=\frac{\normlossout{\shmatrixout_{k\bullet}}}{\normlossout{\shmatrixout_{k\bullet}}}=1\,,
\end{multline*}
the previous equality means that the vector $(\shmatrixout\shfunctionF{\shmatrixin\shmatrixdata})_{k\bullet}\in\R^{\tlnbrdata}$ is a $\losspowerout$-convex combination of the $2\tlnbrmed$  vectors $\caravector_1,\dots,\caravector_{\tlnbrmed},-\caravector_1,\dots,-\caravector_{\tlnbrmed}\in\R^{\tlnbrdata}$, that is,
$(\shmatrixout\shfunctionF{\shmatrixin\shmatrixdata})_{k\bullet}\in   \operatorname{conv}_{\losspower}[\caravector_1,\dots,\caravector_{\tlnbrmed},-\caravector_1,\dots,-\caravector_{\tlnbrmed}]$.
Hence, 
by Lemma~\ref{res:caraold},
there is a $\caraweight\in[-1,1]^{\tlnbrmed}$
such that 
$\normlossout{\caraweight} \leq1$, 
$\card{j\in\{1,\dots,\tlnbrmed\}\,:\,\caraweightj \neq 0}\leq \tlnbrdata+1$,
and 
\begin{equation*}
  \bigl(\shmatrixout\shfunctionF{\shmatrixin\shmatrixdata}\bigr)_{k\bullet}=\sum_{j=1}^{\tlnbrmed}\caraweightj\normlossout{\shmatrixout_{k\bullet }}\bigl(\shfunctionF{\shmatrixin\shmatrixdata}\bigr)_{j\bullet}\,.
\end{equation*}
Hence, 
by the definition of the row vectors,
\begin{equation*}
  \bigl(\shmatrixout\shfunctionF{\shmatrixin\shmatrixdata}\bigr)_{ki}=\sum_{j=1}^{\tlnbrmed}\caraweightj\normlossout{\shmatrixout_{k\bullet }}\bigl(\shfunctionF{\shmatrixin\shmatrixdata}\bigr)_{ji}
\end{equation*}
\figuretwolayer{t}
for all $i\in\{1,\dots,\tlnbrdata\}$, 
and, more generally,
we find for all $a\in\{1,\dots,\tlnbrout\}$ and $i\in\{1,\dots,\tlnbrdata\}$ that
\begin{equation*}
  \bigl(\shmatrixout\shfunctionF{\shmatrixin\shmatrixdata}\bigr)_{ai}=\begin{cases}
\sum_{j=1}^{\tlnbrmed}\caraweightj\normlossout{\shmatrixout_{k\bullet }}\bigl(\shfunctionF{\shmatrixin\shmatrixdata}\bigr)_{ji}~~~~~~&\text{for~}a=k\,;\\
\sum_{j=1}^{\tlnbrmed}\shmatrixout_{aj}\bigl(\shfunctionF{\shmatrixin\shmatrixdata}\bigr)_{ji}~~~~~~&\text{otherwise}\,.
\end{cases}
\end{equation*}
This motivates us to define $\shmatrixoutS\in\R^{\tlnbrout\times \tlnbrmed}$ through
\begin{equation*}
  \shmatrixoutS_{aj}\deq
  \begin{cases}
 \caraweightj\normlossout{\shmatrixout_{k\bullet }}~~~~~~&\text{for~}a=k\,;\\
 \shmatrixout_{aj}~~~~~~&\text{otherwise}\,.
  \end{cases}
\end{equation*}

Properties~1,~3, and~4 then follow immediately. 
Property~2 can be derived by using 
1.~the definition of the basic regularizer~\regularizerR\ on Page~\pageref{def:regularizerbasic},
2.~the definition of~$\shmatrixoutS$,
3.~the linearity of finite sums,
4.~the definition of the $\ell_{\losspower}$-norms on Page~\pageref{def:lq},
5.~the above-derived property $\normlossout{\caraweight} \leq 1$,
6.~a consolidation,
and 7.~again the definition of the basic regularizer~\regularizerR\ on Page~\pageref{def:regularizerbasic}:
\begingroup
\allowdisplaybreaks
\begin{align*}
\bigl(\regularizerRF{\shmatrixoutS,\losspowerout}\bigr)^{\losspowerout}
&=\max_{a\in\{1,\dots,\tlnbrout\}}\sum_{j=1}^{\tlnbrmed}\abs{\shmatrixoutS_{aj}}^{\losspowerout}\\
&=\max_{a\in\{1,\dots,\tlnbrout\}}
\begin{cases}
\sum_{j=1}^{\tlnbrmed}\absB{\caraweightj\normlossout{\shmatrixout_{k\bullet }}}^{\losspowerout}~~~~~~~&\text{for~}a=k\\
\sum_{j=1}^{\tlnbrmed}\abs{\shmatrixout_{aj}}^{\losspowerout}~~~~~~~&\text{otherwise}\\
\end{cases}\\
&=\max_{a\in\{1,\dots,\tlnbrout\}}
\begin{cases}
\normlossoutH{\shmatrixout_{k\bullet }}\sum_{j=1}^{\tlnbrmed}\abs{\caraweightj}^{\losspowerout}~~~~~~~&\text{for~}a=k\\
\sum_{j=1}^{\tlnbrmed}\abs{\shmatrixout_{aj}}^{\losspowerout}~~~~~~~&\text{otherwise}\\
\end{cases}\\
&=\max_{a\in\{1,\dots,\tlnbrout\}}
\begin{cases}
 \bigl(\sum_{j=1}^{\tlnbrmed}\abs{\shmatrixout_{kj}}^{\losspowerout}\bigr)\normlossoutH{\caraweight}~~~~~~~&\text{for~}a=k\\
\sum_{j=1}^{\tlnbrmed}\abs{\shmatrixout_{aj}}^{\losspowerout}~~~~~~~&\text{otherwise}\\
\end{cases}\\
&\leq\max_{a\in\{1,\dots,\tlnbrout\}}
\begin{cases}
\sum_{j=1}^{\tlnbrmed}\abs{\shmatrixout_{kj}}^{\losspowerout}~~~~~~~&\text{for~}a=k\\
\sum_{j=1}^{\tlnbrmed}\abs{\shmatrixout_{aj}}^{\losspowerout}~~~~~~~&\text{otherwise}\\
\end{cases}\\
&=\max_{a\in\{1,\dots,\tlnbrout\}}\sum_{j=1}^{\tlnbrmed}\abs{\shmatrixout_{aj}}^{\losspowerout}\\
&=\bigl(\regularizerRF{\shmatrixout, \losspowerout}\bigr)^{\losspowerout}\,,
\end{align*}
\endgroup
as desired.
This concludes the proof of the first step.

\emph{Step~2:}
We now show that there is a matrix $\widetilde\shmatrixout\in\R^{\tlnbrout\times \tlnbrmed}$ such that
\begin{enumerate}
\item $\widetilde\shmatrixout\shfunctionF{\shmatrixin\shmatrixdata}=\shmatrixout\shfunctionF{\shmatrixin\shmatrixdata}$;
\item $\regularizerRF{\widetilde\shmatrixout,\losspowerout}\leq \regularizerRF{\shmatrixout,\losspowerout}$;
\item $\card{j\in\{1,\dots,\tlnbrmed\}\,:\,\widetilde\shmatrixout_{aj}\neq 0}\leq \tlnbrdata+1$ for all $a\in\{1,\dots,\tlnbrout\}$.
\end{enumerate}
Hence, we replace the  matrix~\shmatrixout\ by a matrix whose every row has at most~$\tlnbrdata+1$ nonzero entries---see the illustration in Figure~\ref{fig:twolayer} again.

Since Step~1 changes only the $k$th row of~$\shmatrixout$ (see Property~4 derived in Step~1),
we can apply it to one row after another.

\emph{Step~3:} We finally prove the first part of the lemma---see the illustration in Figure~\ref{fig:twolayer} again.

By Property~3 of the previous step,
the matrix $\widetilde\shmatrixout$ has at most $\tlnbrout(\tlnbrdata+1)$ nonzero columns.
Verify that replacing~$\shmatrixout$ by~$\shmatrixout_{\perm}$ and~$\shmatrixin$ by~$\shmatrixin^{\perm}$
for a suitable permutation~\perm\ leads to an $\widetilde\shmatrixout$ whose entries outside the first $\tlnbrout(\tlnbrdata+1)$~columns are equal to zero---while all other properties remain intact.
We denote this version of~$\widetilde\shmatrixout$ by $\overline{\shmatrixout}$.
We then derive for all $j\in\{1,\dots,\tlnbrmed\}$ and $i\in\{1,\dots,\tlnbrdata\}$ that
\begin{multline*}
  \bigl(\shfunctionF{\shmatrixin^{\perm}\shmatrixdata}\bigr)_{ji}
=\shfunctionFB{(\shmatrixin^{\perm}\shmatrixdata)_{ji}}
=\shfunctionFBBBB{\sum_{b=1}^{\tlnbrin}(\shmatrixin^{\perm})_{jb}\shmatrixdata_{bi}}
=\shfunctionFBBBB{\sum_{b=1}^{\tlnbrin}(\shmatrixdata\tp)_{ib}(\shmatrixin^{\perm})_{jb}}\\
=\shfunctionFB{(\shmatrixdata\tp(\shmatrixin^{\perm})_{j\bullet})_i}=\bigl(\shfunctionF{\shmatrixdata\tp(\shmatrixin^{\perm})_{j\bullet}}\bigr)_i\,,
\end{multline*}
where 
we define (with some abuse of notation)  $\shfunction\,:\,\R^{\tlnbrdata}\to\R^{\tlnbrdata}$ through $(\shfunctionF{\boldsymbol{b}})_{i}\deq \shfunctionF{b_{i}}$ for all $\boldsymbol{b}\in\R^{\tlnbrdata}$.
Combining this result with the results of Step~2 (with $\shmatrixout$ and $\shmatrixin$ replaced by~$\shmatrixout_{\perm}$ and~$\shmatrixin^{\perm}$, respectively)  yields for all $a\in\{1,\dots,\tlnbrout\}$ and $i\in\{1,\dots,\tlnbrdata\}$ that
\begin{multline*}
  \bigl(\shmatrixout_{\perm}\shfunctionF{\shmatrixin^{\perm}\shmatrixdata}\bigr)_{ai}=\sum_{j=1}^{\tlnbrmed}(\shmatrixout_{\perm})_{aj}\bigl(\shfunctionF{\shmatrixin^{\perm}\shmatrixdata}\bigr)_{ji}=\sum_{j=1}^{\tlnbrmed}(\shmatrixout_{\perm})_{aj}\bigl(\shfunctionF{\shmatrixdata\tp(\shmatrixin^{\perm})_{j\bullet}}\bigr)_i\\=\sum_{j=1}^{\tlnbrmed}\overline\shmatrixout_{aj}\bigl(\shfunctionF{\shmatrixdata\tp(\shmatrixin^{\perm})_{j\bullet}}\bigr)_i=\sum_{j=1}^{\min\{\tlnbrout(\tlnbrdata+1),\tlnbrmed\}}\overline\shmatrixout_{aj}\bigl(\shfunctionF{\shmatrixdata\tp(\shmatrixin^{\perm})_{j\bullet}}\bigr)_i\,.
\end{multline*}
We then define $\shmatrixinU\in\R^{\tlnbrmed \times \tlnbrin}$ through
\begin{equation*}
  \overline{\shmatrixin}_{ji}\deq
                     \begin{cases}
(\shmatrixin^{\perm})_{ji}=\shmatrixin_{\permF{j}i}~~~~~&\text{for}~j\leq \tlnbrout(\tlnbrdata+1)\,;\\
                       0~~~~~&\text{otherwise}\,.                       
                     \end{cases}
\end{equation*}
We then use
1.~the above-stated equality,
2.~the definition of~\shmatrixinU,
3.~a similar derivation as above,
4.~the block structure of~\shmatrixoutU,
and 5.~a similar derivation as above to establish for all $a\in\{1,\dots,\tlnbrout\}$ and $i\in\{1,\dots,\tlnbrdata\}$ the fact that
\begin{align*}
  \bigl(\shmatrixout_{\perm}\shfunctionF{\shmatrixin^{\perm}\shmatrixdata}\bigr)_{ai}
&=\sum_{j=1}^{\min\{\tlnbrout(\tlnbrdata+1),\tlnbrmed\}}\shmatrixoutU_{aj}\bigl(\shfunctionF{\shmatrixdata\tp(\shmatrixin^{\perm})_{j\bullet}}\bigr)_i\\
&=\sum_{j=1}^{\min\{\tlnbrout(\tlnbrdata+1),\tlnbrmed\}}\shmatrixoutU_{aj}\bigl(\shfunctionF{\shmatrixdata\tp\shmatrixinU_{j\bullet}}\bigr)_i\\
&=\sum_{j=1}^{\min\{\tlnbrout(\tlnbrdata+1),\tlnbrmed\}}\shmatrixoutU_{aj}\bigl(\shfunctionFB{\shmatrixinU\shmatrixdata}\bigr)_{ji}\\
&=\sum_{j=1}^{\tlnbrmed}\shmatrixoutU_{aj}\bigl(\shfunctionFB{\shmatrixinU\shmatrixdata}\bigr)_{ji}\\
&=\bigl(\shmatrixoutU\shfunctionF{\shmatrixinU\shmatrixdata}\bigr)_{ai}\,.
\end{align*}
The other properties stated in the lemma follow readily.

The second part of the lemma can be derived in the same way.
\end{proof}

\subsection{Proof of Lemma~\ref{res:convex}}

\newcommand{\figureconvex}[1]{\begin{figure}[#1]
  \centering
  \begin{tikzpicture}[scale=1]
    % Axes 
    \draw[-latex] (-0.1, 0) -- (3.5, 0);
    \draw[-latex] (0, -0.1) -- (0, 2);
    % Axes labels
   \node at (3.3, -0.2) {\scriptsize \paradd};
   \node at (0.7, 0.5) {\scriptsize $\pathfunctionF{\paradd}$};
   \node at (1.65, 0.8) {\scriptsize $\pathfunctionPF{\paradd}$};

   % Connections to important points
   \draw[very thin] (2, 0.25) -- (2, 0);
   \draw[very thin] (2.7, 0.415) --(2.7, 0);
   \draw[very thin] (2, 0.25) --(3, 0.25);

    % Axes ticks
    \draw (3, 0.1) -- (3, -0.1);
    \node at (0, -0.3) {\scriptsize $0$};
    \node at (3, -0.3) {\scriptsize $1$};
    \node at (2, -0.2) {\scriptsize $\paraddPP$};
    \node at (2.7, -0.2) {\scriptsize $\paraddP$};

    % Functions
    \draw[scale=1, domain=0:3, smooth, variable=\x,  thick] plot ({\x}, {(\x-2)*(\x-2)/3+0.25});
    \draw[scale=1, domain=0:3, smooth, variable=\x, thick] plot ({\x}, {(\x/1.55-2)*(\x/1.55-2)/3+0.25});

    % Important points
    \filldraw[] (0, 1.58) circle (1pt);
    \filldraw[] (3, 0.58) circle (1pt);
    \filldraw[] (3, 0.25) circle (1pt);
    \filldraw[] (2, 0.25) circle (1pt);
    \filldraw[] (2.7, 0.415) circle (1pt);
  \end{tikzpicture}
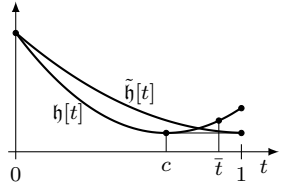
  \caption{quantities in the proof of  Lemma~\ref{res:convex}}
  \label{fig:convex}
\end{figure}}

\figureconvex{t}
\begin{proof}[Proof of Lemma~\ref{res:convex}]
The proof is a simple exercise in calculus.
An illustration of the quantities involved in the proof  is given in Figure~\ref{fig:convex}.
The function~\pathfunction\ is convex by assumption;
hence, it is continuous.
Then, 
according to the extreme value theorem,
there is a number
\begin{equation*}
  \paraddPP\in\argmin_{\paradd\in[0,1]}\bigl\{\pathfunctionF{\paradd}\bigr\}\,.
\end{equation*}
Since $\pathfunctionF{\paraddP}<\pathfunctionF{0}$ for a $\paraddP\in(0,1]$,
it holds that $\paraddPP\in(0,1]$.
Now consider $\paraddt\in[0,1]$ and $\paraddo\in[0,\paraddt)$, that is, $\paraddt>\paraddo$.
Basic calculus ensures that
\begin{align*}
  \pathfunctionPF{\paraddt}&=\pathfunctionF{\paraddPP\paraddt}\\
&=\pathfunctionFB{\paraddPP\paraddo+\paraddPP\paraddt-\paraddPP\paraddo}\\
&=\pathfunctionFBBB{\paraddPP\paraddo+\frac{\paraddPP\paraddt-\paraddPP\paraddo}{\paraddPP-\paraddPP\paraddo}(\paraddPP-\paraddPP\paraddo)}\\
&=\pathfunctionFBBB{\paraddPP\paraddo+\frac{\paraddt-\paraddo}{1-\paraddo}(\paraddPP-\paraddPP\paraddo)}\\
&=\pathfunctionFBBB{\Bigl(1- \frac{\paraddt-\paraddo}{1-\paraddo}\Bigr)\paraddPP\paraddo+\frac{\paraddt-\paraddo}{1-\paraddo}\paraddPP}\\
&\leq\Bigl(1- \frac{\paraddt-\paraddo}{1-\paraddo}\Bigr)\pathfunctionF{\paraddPP\paraddo}+\frac{\paraddt-\paraddo}{1-\paraddo}\pathfunctionF{\paraddPP}\\
&\leq\Bigl(1- \frac{\paraddt-\paraddo}{1-\paraddo}\Bigr)\pathfunctionF{\paraddPP\paraddo}+\frac{\paraddt-\paraddo}{1-\paraddo}\pathfunctionF{\paraddPP\paraddo}\\
&=\pathfunctionF{\paraddPP\paraddo}\\
&= \pathfunctionPF{\paraddo}\,.
\end{align*}
Hence, $\pathfunctionP$ is nonincreasing.

Moreover,
\begin{equation*}
  \pathfunctionPF{0}=\pathfunctionF{\paraddPP\cdot 0}=\pathfunctionF{0}>\pathfunctionF{\paraddP} \geq \pathfunctionF{\paraddPP}=\pathfunctionF{\paraddPP\cdot 1}=\pathfunctionPF{1}\,.
\end{equation*}
 Hence,
$\pathfunctionPF{0}>\pathfunctionPF{1}$.
This concludes the proof.
\end{proof}

\end{document}